\documentclass[10pt,twocolumn,letterpaper]{article}

\usepackage[pagenumbers]{cvpr} 
\makeatletter
\@namedef{ver@everyshi.sty}{}
\makeatother
\usepackage{tikz}
\usepackage{comment}
\usepackage{graphicx}
\usepackage{booktabs}
\usepackage{amsfonts}
\usepackage{bbding}
\usepackage{algorithmic}
\usepackage{algorithm}
\usepackage{array}
\usepackage[title]{appendix}
\usepackage{textcomp}
\usepackage{stfloats}
\usepackage{url}
\usepackage{verbatim}
\usepackage{graphicx}
\usepackage{soul}
\usepackage{url}
\usepackage[utf8]{inputenc}
\RequirePackage{caption}
\usepackage{amssymb,amsmath,amsthm}
\newtheorem{theorem}{Theorem}
\usepackage{mathtools} 
\usepackage{booktabs}
\usepackage{amsthm}
\usepackage{tikz}
\usepackage{amsmath,amssymb} 
\usepackage{color}
\usepackage[pagebackref,breaklinks,colorlinks]{hyperref}
\usepackage[capitalize]{cleveref}
\crefname{section}{Sec.}{Secs.}
\Crefname{section}{Section}{Sections}
\Crefname{table}{Table}{Tables}
\crefname{table}{Tab.}{Tabs.}
\def\confName{CVPR}
\def\confYear{2023}
\pdfoutput=1
\begin{document}
	\title{Weakly-supervised Single-view Image Relighting}
	\author{
		Renjiao Yi\thanks{Co-first authors.}, Chenyang Zhu\footnotemark[1],  Kai Xu\thanks{Corresponding author: kevin.kai.xu@gmail.com. }\\
		National University of Defense Technology\\
	}
\maketitle
\begin{abstract}
	We present a learning-based approach to relight a single image of Lambertian and low-frequency specular objects. 
	Our method enables inserting objects from photographs into new scenes and relighting them under the new environment lighting, which is essential for AR applications. To relight the object, we solve both inverse rendering and re-rendering. 
	To resolve the ill-posed inverse rendering, we propose a weakly-supervised method by a low-rank constraint.  
	To facilitate the weakly-supervised training, we contribute Relit, a large-scale (750K images) dataset of videos with aligned objects under changing illuminations.
	For re-rendering, we propose a differentiable specular rendering layer to render low-frequency non-Lambertian materials under various illuminations of spherical harmonics. 
	The whole pipeline is end-to-end and efficient, allowing for a mobile app implementation of AR object insertion. Extensive evaluations demonstrate that our method achieves state-of-the-art performance. Project page: \href{https://renjiaoyi.github.io/relighting/}{https://renjiaoyi.github.io/relighting/}. 
\end{abstract}
\section{Introduction}

Object insertion finds extensive applications in Mobile AR.
Existing AR object insertions require a perfect mesh of the object being inserted. Mesh models are typically built by professionals and are not easily accessible to amateur users. Therefore, in most existing AR apps such as SnapChat and Ikea Place, users can use only built-in virtual objects for scene augmentation.
This may greatly limit user experience.
A more appealing setting is to allow the user to extract objects from a photograph and insert them into the target scene with proper lighting effects. This calls for a method of inverse rendering and relighting based on a single image, which has so far been a key challenge in the graphics and vision fields.

Relighting real objects requires recovering lighting, geometry and materials which are intertwined in the observed image; it involves solving two problems, inverse rendering~\cite{patow2003survey} and re-rendering. Furthermore, to achieve realistic results, the method needs to be applicable for non-Lambertian objects. In this paper, we propose a pipeline to solve both problems, weakly-supervised inverse rendering and non-Lambertian differentiable rendering for Lambertian and low-frequency specular objects. 

\begin{figure}
	\centering
	\includegraphics[width=\linewidth]{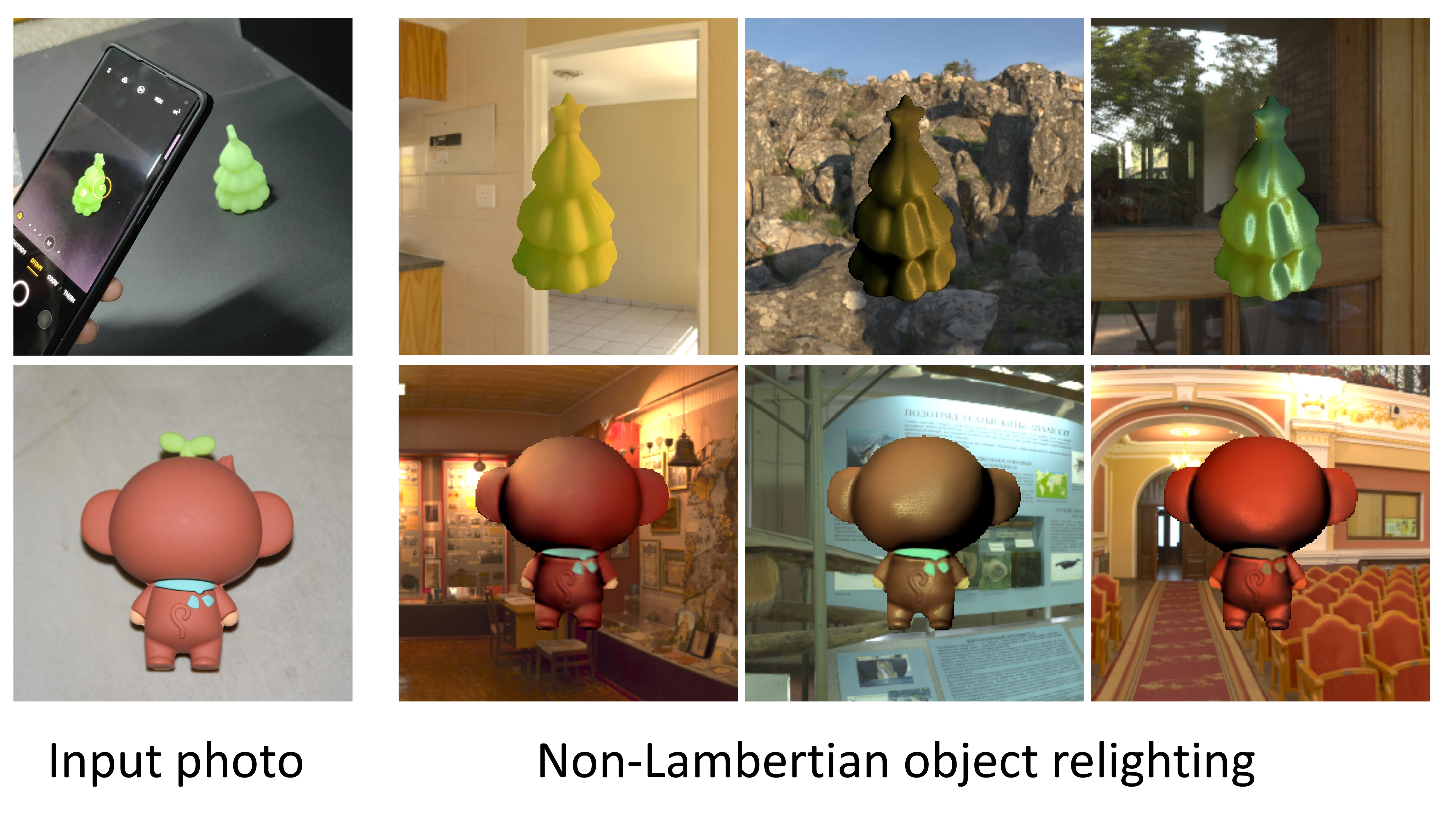}\
	\caption{Our method relights real objects into new scenes from single images, which also enables editing materials from diffuse to glossy with non-Lambertian rendering layers. }
	\label{fig:teaser}
	\vspace{-1.2em}
\end{figure}

\begin{figure*}[t]
	\centering
	\includegraphics[width=\linewidth]{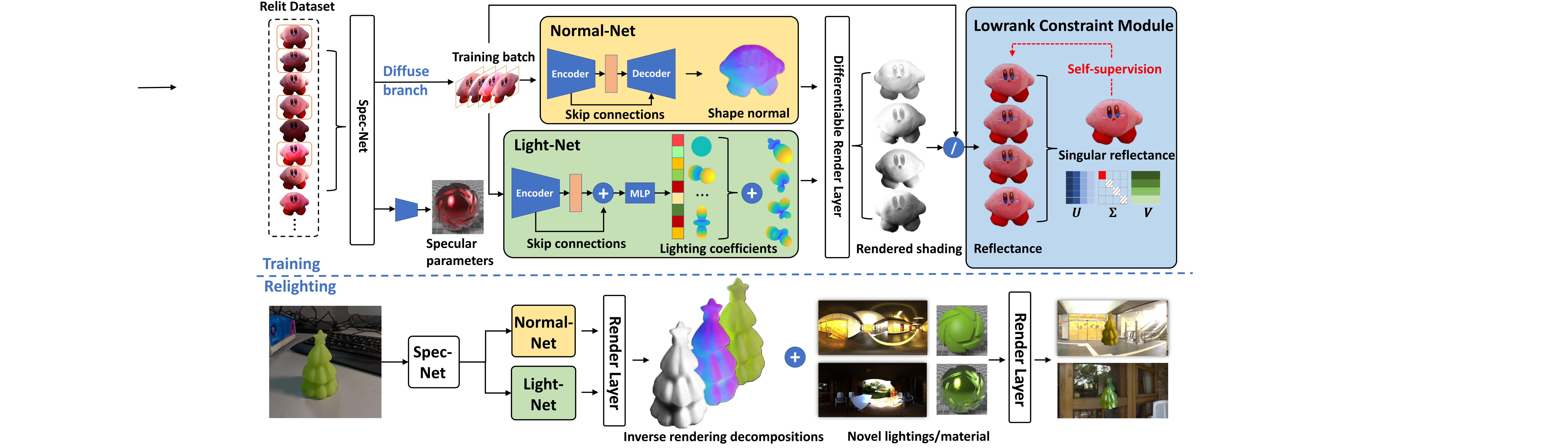}\
	\caption{Overview of our method. At training time, Spec-Net separates input images into specular and diffuse branches. Spec-Net, Normal-Net and Light-Net are trained in a self-supervised manner by the Relit dataset. At inference time, inverse rendering properties are predicted to relight the object under novel lighting and material. The non-Lambertian render layers produce realistic relit images. }
	\label{fig:pipeline}
	\vspace{-0.2cm}
\end{figure*}

Inverse rendering is a highly ill-posed problem, with several unknowns to be estimated from a single image. Deep learning methods excel at learning strong priors for reducing ill-posedness. However, this comes at the cost of a large amount of labeled training data, which is especially cumbersome to prepare for inverse rendering since ground truths of large-scale real data are impossible to obtain. Synthetic training data brings the problem of domain transfer. Some methods explore self-supervised pipelines and acquire geometry supervisions of real data from 3D reconstruction by multi-view stereo (MVS)~\cite{yu2019inverserendernet,yu2020self}. Such approaches, however, have difficulties in handling textureless objects.

To tackle the challenge of training data shortage, 
we propose a \emph{weakly-supervised inverse rendering pipeline} based on a novel low-rank loss and a re-rendering loss. 
For low-rank loss, a base observation here is that the material reflectance is invariant to illumination change, as an intrinsic property of an object. We derive a low-rank loss for inverse rendering optimization which imposes that \emph{the reflectance maps of the same object under changing illuminations are linearly correlated}. 
In particular, we constrain the reflectance matrix with each row storing one of the reflectance maps to be rank one.
This is achieved by minimizing a low-rank loss defined as the Frobenius norm between the reflectance matrix and its rank-one approximation.
We prove the convergence of this low-rank loss. In contrast, traditional Euclidean losses lack a convergence guarantee.

To facilitate the learning, we contribute Relit, a large-scale dataset of videos of real-world objects with changing illuminations. We design an easy-to-deploy capturing system: a camera faces toward an object, both placed on top of a turntable. Rotating the turntable will produce a video with the foreground object staying still and the illumination changing. To extract the foreground object from the video, manual segmentation of the first frame suffices since the object is aligned across all frames.

As shown in Figure~\ref{fig:pipeline}, a fixed number of images under different lighting are randomly selected as a batch. We first devise a Spec-Net to factorize the specular highlight, trained by the low-rank loss on the chromaticity maps of diffuse images (image subtracts highlight) which should be consistent within the batch. With the factorized highlight, we further predict the shininess and specular reflectance, which is self-supervised with the re-rendering loss of specular highlight. For the diffuse branch, we design two networks, Normal-Net and Light-Net, to decompose the diffuse component by predicting normal maps and spherical harmonic lighting coefficients, respectively. The diffuse shading is rendered by normal and lighting, and diffuse reflectance (albedo) is computed by diffuse image and shading. Both networks are trained by low-rank loss on diffuse reflectance. 

Regarding the re-rendering phase, 
the main difficulty is the missing of 3D information of the object given a single-view image. The Normal-Net produces a normal map which is a partial 3D representation, making the neural rendering techniques and commercial renderers inapplicable. The existing diffuse rendering layer for normal maps of~\cite{ramamoorthi2001efficient} cannot produce specular highlights. Pytorch3D and \cite{li2022phyir,li2020inverse} render specular highlights for point lights only. 

To this end,
we design a \emph{differentiable specular renderer} from normal maps, based on the Blinn-Phong specular reflection~\cite{blinn1977models} and spherical harmonic lighting~\cite{green2003spherical}. Combining with the differentiable diffuse renderer, we can render low-frequency non-Lambertian objects with prescribed parameters under various illuminations, and do material editing as byproduct.

We have developed \emph{an Android app} based on our method which allows amateur users to insert and relight arbitrary objects extracted from photographs in a target scene. Extensive evaluations on inverse rendering and image relighting demonstrate the state-of-the-art performance of our method. 
%
%
%

Our contributions include:
\begin{itemize}
	\vspace{-5pt}\item A weakly-supervised inverse rendering pipeline trained with a low-rank loss. The correctness and convergence of the loss are mathematically proven.
	\vspace{-5pt}\item A large-scale dataset of foreground-aligned videos collecting $750K$ images of $100$+ real objects under different lighting conditions.
	\vspace{-5pt}\item An Android app implementation for amateur users to make a home-run.
\end{itemize}
\section{Related Work}

{\bf{Inverse rendering.}}
As a problem of inverse graphics, inverse rendering aims to solve geometry, material and lighting from images. This problem is highly ill-posed. Thus some works tackle the problem by targeting a specific class of objects, such as faces~\cite{shu2017neural,tewari2017mofa} or planar surfaces~\cite{aittala2016reflectance}. For inverse rendering of general objects and scenes, most prior works~\cite{barron2015shape,janner2017self,li2018learning2,Lichy_2021_CVPR} require direct supervisions by synthesized data. However, networks trained on synthetic data have a domain gap for real testing images. Ground truths of real images are impossible to obtain, and it calls for self-supervised methods training on real images. Recently, self-supervised methods~\cite{yu2019inverserendernet,yu2020self} explore self-supervised inverse rendering for outdoor buildings, where the normal supervision is provided by reconstructing the geometry by MVS. However, they do not work well for general objects, which is reasonable because object images are unseen during training. However, applying the pipelines for objects meet new problems. Textureless regions on objects are challenging for MVS due to lack of features. It motivates our work on weakly-supervised inverse rendering for general objects. To fill the blank of real-image datasets on this topic, we capture a large-scale real-image datasets Relit to drive the training.  

There are also many works addressing inverse rendering as several separated problems, such as intrinsic image decomposition~\cite{shi2017learning,yi2020leveraging,liu2020unsupervised,li2018cgintrinsics}, specularity removal~\cite{shen2013real,shi2017learning,yamamoto2019general} or surface normal estimation~\cite{li2018learning2}. In order to compare with more related methods, we also evaluate these tasks individually in experiments. 

{\bf{Image relighting.}}
Most prior methods in image-based relighting require multi-image inputs~\cite{azinovic2019inverse,xu2018deep}. For example, in \cite{xu2018deep}, a scene is relit from a sparse set of five images under the optimal light directions predicted by CNNs. Single-image relighting is highly ill-posed, and needs priors. \cite{philip2019multi,yu2020self} target outdoor scenes, and benefit from priors of outdoor lighting models. \cite{meka2019deep,shu2017neural,shu2017portrait,sun2019single,sengupta2018sfsnet} target at portrait images, which is also a practical application for mobile AR. Single image relighting for general scenes have limited prior works. Yu et al. \cite{yu2020self} takes a single image as inputs, with the assumption of Lambertian scenes. In this work, we propose a novel non-Lambertian render layer, and demonstrate quick non-Lambertian relighting of general objects.  

\section{Overview}

We propose a deep neural network to solve single-image inverse rendering and object-level relighting. 
The overall pipeline is shown in Figure~\ref{fig:pipeline}. 
The whole pipeline is weakly-supervised with a supervised warm-up of Normal-Net, and self-supervised training of the whole pipeline. The self-supervised training is driven by the Relit Dataset. 
The details of Relit Dataset is intoduced in Section~\ref{sec:dataset}. 
In Section~\ref{sec:method}, we introduce the proposed pipeline following the order from single-image inverse rendering to differentiable non-Lambertian relighting. The weakly-supervised inverse rendering, including the proofs of theoretical fundamentals and convergence of the low-rank loss, are introduced in Section~\ref{sec:network}. 
The differentiable non-Lambertian rendering layers are introduced in Section~\ref{sec:relighting}. 

\section{The Relit Dataset}\label{sec:dataset}
\begin{figure*}[t]
	\centering
	\includegraphics[width=0.95 \linewidth]{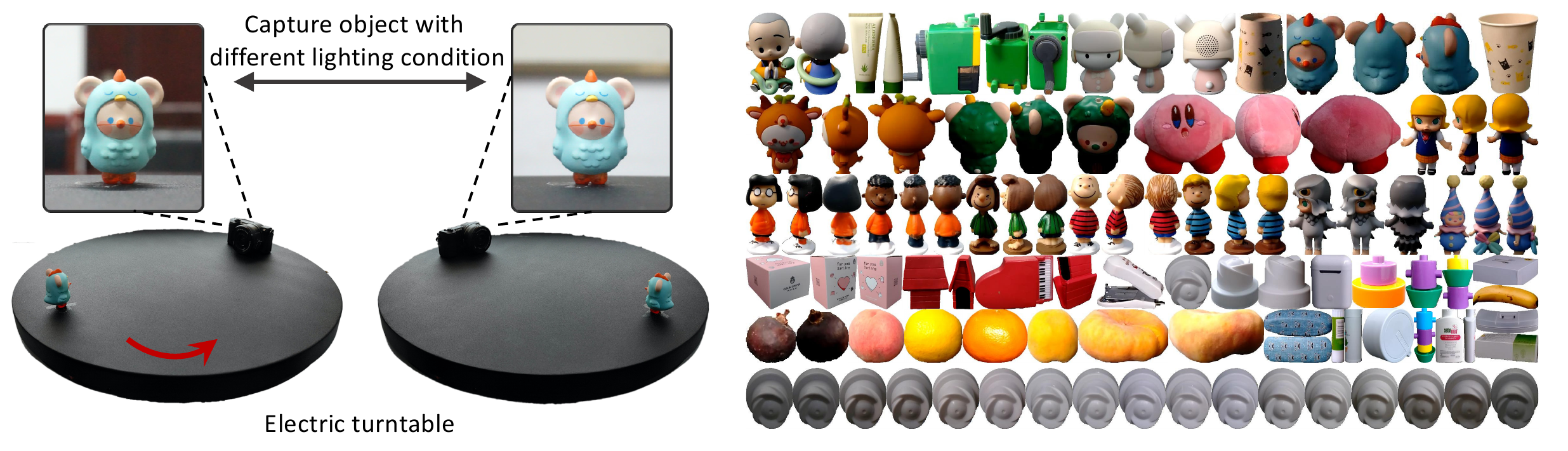}
	\caption{Left: The data capture set-up. Right: Selected objects in Relit dataset. The last row shows selected frames from one video. }
	\label{fig:dataset}
\end{figure*}
To capture foreground-aligned videos of objects under changing illuminations, we design an automatic device for data capture, as shown in Figure~\ref{fig:dataset} (left). The camera and object are placed on the turntable, and videos are captured as the turntable rotating. The target object stays static among the frames in captured videos, with changing illuminations and backgrounds. 
In summary, the Relit dataset consists of 500 videos for more than 100 objects under different indoor and outdoor lighting. Each video is 50 seconds, resulting in 1500 foreground-aligned frames under various lighting. In total, the Relit dataset consists of $750K$ images. 
Selected objects are shown in Figure~\ref{fig:dataset} (right). The objects cover a wide variety of shapes, materials, and textures. In Section~\ref{sec:method}, we introduce how to leverage Relit dataset to drive the self-supervised training. 
It can facilitate many tasks, such as image relighting and segmentation. 

\section {Our Method}\label{sec:method}

\subsection{Image formation model}

A coarse-level image formation model for inverse rendering is intrinsic image decomposition (IID), which is a long-standing low-level vision problem, decomposing surface reflectance from other properties, assuming Lambertian surfaces. For non-Lambertian surfaces, the model can be improved by adding a specular highlight term:

\begin{equation}
	I= I_d + H, \quad I_d = \mathcal{A}\odot S, \label{equation:iid}
\end{equation}

\noindent where $H$ is the specular highlight, $\mathcal{A}$ is the surface reflectance map, i.e. albedo map in IID, and $S$ is a term describing the shading related to illumination and geometry. Here $\odot$ denotes the Hadamard product. To be more specific, according to the well-known Phong model\cite{phong1975illumination} and Blinn-Phong model\cite{blinn1977models}, the image can be formulated as the sum of a diffuse term and a specular term:

\begin{equation}
	\centering
	\begin{aligned}
		I(p)=&I_d(p)+H(p),\\
		I_d(p)=&\mathcal{A}(p)S(p)=\mathcal{A}(p)\sum_{\omega \in \mathcal{L}}l_\omega(L_\omega\cdot n(p)),\\ H(p)=&\sum_{\omega \in \mathcal{L}}s_pl_\omega(\frac{L_\omega+v}{\|L_\omega+v\|}\cdot n(p))^\alpha,\label{equation:phong}
	\end{aligned}
\end{equation}

\noindent where $I(p)$ is the observed intensity and $n_p=(x,y,z)$ is the surface normal at pixel $p$. $\mathcal{L}$ is a set of sampled point lights in the lighting environment. $L_\omega$ and $l_\omega$ describe lighting direction and intensity of one point light $\omega$ in $\mathcal{L}$ respectively. $\mathcal{A}(p)$ and $s_p$ are defined as the diffuse and specular reflectance at pixel $p$, respectively. The specular term is not view independent, view direction $v$ is needed to calculate the reflectance intensity and $\alpha$ is a shininess constant. The differentiable approximation for Equation~(\ref{equation:phong}) is introduced in Section~\ref{sec:diffuserender}-\ref{sec:specrender}.

\subsection{Inverse rendering from a single image}\label{sec:network}

For relighting, we first inverse the rendering process to get 3D properties including geometry, reflectance, shading, illumination and specularities, then we can replace the illumination and re-render the objects. Following this order, we firstly introduce inverse rendering. 

For non-Lambertian object, we can perform specular highlight separation first by the Spec-Net. The specular parameters are then predicted in the specular branch, which is introduced in Section~\ref{sec:spec}. 


For diffuse branch, adopting separate networks to predict normal, lighting, shading, reflectance is the most straightforward choice. However, in this way, the diffuse component in the rendering equation (Equation (\ref{equation:phong})) is not respected, since relations between these properties are not constrained. 
Thus, we design a lightweight physically-motivated inverse rendering network, respecting the rendering equation strictly, as shown in Figure~\ref{fig:pipeline}. 
There are only two learnable network modules in our end-to-end diffuse inverse rendering pipeline.
Here we adopt spherical harmonics~\cite{ramamoorthi2001efficient} to represent illumination $\mathcal{L}$ in Equation (\ref{equation:phong})), which is calculated more efficiently than Monte Carlo integration of point lights:

\begin{equation}
\mathcal{L}=\sum_{l=0}^{\infty}\sum_{m=-l}^{l}C_{l,m}Y_{l,m}, \label{equation:SH}
\end{equation}

\noindent where $Y_{l,m}$ is the spherical harmonic basis of degree $l$ and order $m$, $C_{l,m}$ is the corresponding coefficient. Each environment lighting can be represented as the weighted sum of spherical harmonics. The irradiance can be well approximated by only 9 coefficients, 1 for $l = 0, m = 0$, 3 for $l = 1, -1 \leq m \leq 1$, and 5 for $l = 2, -2 \leq m \leq 2$. 

Normal-Net predicts surface normal maps $n$, and Light-Net regresses lighting coefficients $C_{l,m}$ in spherical harmonic representation. A total of 12 coefficients are predicted by Light-Net, where the last 3 coefficients present the illumination color. 
The shading $S$ is then rendered from the predicted normal and lighting, by a hard-coded differentiable rendering layer (no learnable parameters) in Section~\ref{sec:diffuserender}, following Equation (\ref{equation:phong}). The reflectance $\mathcal{A}$ is computed by Equation (\ref{equation:iid}) after rendering shading. The pipeline design is based on the physical rendering equation (Equation (\ref{equation:phong})), where relations among terms are strictly preserved. 




\subsubsection{Self-supervised low-rank constraint}\label{sec:unsupervised}


We have foreground-aligned videos of various objects under changing illuminations in Relit dataset. The target object is at a fixed position in each video, which enables pixel-to-pixel losses among frames. 

For each batch, $N$ images $I_1$, $I_2$, ..., $I_N$ are randomly selected from one video. Since the object is aligned in $N$ images under different lighting, one observation is that the reflectance should remain unchanged as an intrinsic property, and the resulting reflectance $\mathcal{A}_1$, $\mathcal{A}_2$, ..., $\mathcal{A}_N$ should be identical. However, due to the scale ambiguity between reflectance and lighting intensities, i.e., estimating reflectance as $\mathcal{A}$ and lighting as $\mathcal{L}$, is equivalent to estimating them as $w \mathcal{A}$ and $\frac{1}{w}\mathcal{L}$. A solution for supervised methods is defining a scale-invariant loss between ground truths and predictions. However the case is different here, there are no predefined ground truths. While adopting traditional Euclidean losses between every pair in $\mathcal{A}_1$, $\mathcal{A}_2$, ..., $\mathcal{A}_N$, it leads to degenerate results where all reflectance are converged to zero. To solve the problem, here we enforce $\mathcal{A}_1$, $\mathcal{A}_2$,..., $\mathcal{A}_N$ to be linearly correlated and propose a rank constraint as training loss. Therefore, a scaling factor $w$ does not affect the loss. 

We can compose a matrix $R$ with each reflectance $\mathcal{A}_i$ storing as one row. Ideally, rows in $R$ should be linearly correlated, i.e., $R$ should be rank one. 
We formulate a self-supervised loss by the distance between $R$ and its rank-one approximation. 
We introduce Theorem 1 below. 

\begin{theorem}\label{theorem:rankone}
\textbf{Optimal rank-one approximation.} By SVD,  $R = U\Sigma V^T$,  $\Sigma=diag(\sigma_1,\sigma_2,...\sigma_k)$, $\Sigma'=diag(\sigma_1,0,...)$, $\bar{R} = U\Sigma'V^T$ is the optimal rank-one approximation for R, which meets:

\begin{equation}
	\label{equation:optimal}
	||\bar{R}-R||^2_F = \min_{\scriptscriptstyle b\in\mathcal{R}^{N},c\in\mathcal{R}^{d}}||bc^T-R||^2_F,
\end{equation}

\noindent where $||\cdot||_F$ denotes the Frobenius norm of a matrix. 

\end{theorem}

The proof of Theorem~\ref{theorem:rankone} can be found in Appendix~\ref{sec:supp}. 

Therefore, we define the low-rank loss as:

\begin{equation}
f(R) = ||\bar{R}-R||^2_F. 
\end{equation} \label{equation:loss}

Its convergence is proven as below, fitting the needs of learning-based approaches training by gradient descents. 

Since the gradient of $\bar{R}$ is detached from the training, the derivative of $f(R)$ can be accomplished as $\nabla f(R)=-2(\bar{R}-R)$. According to the gradient descent algorithm, with a learning rate  $\eta$, the result $R^{(n+1)}$ ($R$ after $n+1$ training iterations), can be deduced as: 

\begin{equation}
\label{equation:loss_gd}
R^{(n+1)} = R^{(n)} + 2\eta(\bar{R}-R^{(n)}),
\end{equation} 


\begin{theorem}
\textbf{Convergence of $f(R)$.} 
The loss $f(R)$ would converge to a fixed point, which is $\bar{R}$ while $0<\eta<0.5$, 
\begin{equation}
	\label{equation:theorem2}
	\lim_{n \to \infty}R^{(n)}=\bar{R} \Leftarrow  0<\eta<0.5, R^{(0)}=R. 
\end{equation} 
\end{theorem}

\begin{proof}
According to Equation (\ref{equation:loss_gd}), $R = U\Sigma V^T$ and $\bar{R} = U\Sigma' V^T$, we have:
\begin{equation}
	\label{equation:r1}
	\begin{aligned}
		R^{(1)}&=R+2\eta(\bar{R}-R)\\
		&=U \mbox{diag} \{\sigma_1, (1-2\eta)\sigma_2,\dots, (1-2\eta)\sigma_k\}V^T.
	\end{aligned}
\end{equation} 

\noindent Since $0<\eta<0.5$, we have $1-2\eta < 1$, $\{\sigma_1, (1-2\eta)\sigma_2,\dots, (1-2\eta)\sigma_k\}$ are still descending. Therefore, Equation (\ref{equation:r1}) is the SVD form for $R^{(1)}$. Similarly, we have: 
\vspace{-0.2cm}
\begin{equation}
	\label{equation:r2}
	\begin{array}{l}
		R^{(2)}=U \mbox{diag} \{\sigma_1, (1-2\eta)^2\sigma_2,\dots, (1-2\eta)^2\sigma_k\}V^T. 
	\end{array}
\end{equation}

\noindent Repeat $n$ iterations, we have the expression for $R^{(n)}$: 
\vspace{-0.2cm}
\begin{equation}
	\label{equation:rn}
	R^{(n)}=U \mbox{diag} \{\sigma_1, (1-2\eta)^n\sigma_2,\dots, (1-2\eta)^n\sigma_k\}V^T. 
\end{equation}

\noindent Since $|1-2\eta|<1$, Equation (\ref{equation:rn}) can be reduced to:
\begin{equation}
	\label{equation:rn_final}
	\begin{aligned}
		\lim_{n \to \infty}R^{(n)}&=U \mbox{diag} \{\sigma_1, 0,\dots, 0\}V^T\\
		&=U\Sigma' V^T=\bar{R}. 
	\end{aligned}
\end{equation}
\vspace{-0.5cm}
\end{proof}

In our diffuse branch, the low-rank loss of reflectance back-propagates to Normal-Net and Light-Net, and trains both in self-supervised manners. 
\subsubsection{Specularity separation}\label{sec:spec}

To deal with the specular highlights, we add a Spec-Net, to remove the highlights before diffuse inverse rendering. On highlight regions, pixels are usually saturated and tends to be white. Based on it, we automatically evaluate the percentage of saturated pixels on the object image. If the percentage exceeds $5\%$, Spec-Net will be performed, otherwise the object is considered as diffuse and Spec-Net will not be performed. We found that under this setting the results are better than performing Spec-Net on all images, since learning-based highlight removal methods tend to overextract highlights on diffuse images.  
The training of Spec-Net is initialized from the highlight removal network of Yi et al. \cite{yi2020leveraging}, enhanced with images of non-Lambertian objects in our Relit Dataset by self-supervised finetuning. From the Di-chromatic reflection model~\cite{shafer1985color}, if illumination colors remain unchanged, the rg-chromaticity of Lambertian reflection should be unchanged as well. Thus the finetuning can be driven by the low-rank constraint on rg-chromaticity of diffuse images after removing specular highlights, following the image formation model in Equation~(\ref{equation:iid}). 

With the separated specular highlight, we can further predict specular reflectance $s_p$ and shininess (smoothness) $\alpha$ in Equation~(\ref{equation:phong}). The training is self-supervised by re-rendering loss between the separated highlight by Spec-Net, and the re-rendered specular highlight by the predicted $s_p$, $\alpha$, lighting coefficients $C_{l,m}$ from Light-Net via the specular rendering layer in Section~\ref{sec:specrender}. 

\subsubsection{Joint training}


Firstly, the Spec-Net is trained to separate input images into specular highlight and diffuse images, as the first phase. Since training to predict specular reflectance and smoothness requires lighting coefficients from Light-Net, Light-Net and Normal-Net in the diffuse branch are trained as the second phase. Training to predict specular reflectance and smoothness is the last phase. 

In the second phase, Light-Net predicts spherical harmonic lighting coefficients $C_{l,m}$ corresponding to each basis $Y_{l,m}$. There is an axis ambiguity between Normal-Net and Light-Net predictions. For example, predicting a normal map with the $x$-axis pointing right with positive coefficients of the bases related to $x$, is equivalent to predicting a normal map with $x$-axis pointing left with corresponding coefficients being negative. They would render the same shading results. Normal-Net and Light-Net are in a chicken-and-egg relation and cannot be tackled simultaneously.
We employ a joint training scheme to train Normal-Net and Light-Net alternatively. To initialize the coordinate system in Normal-Net, we use a small amount of synthetic data (50k images) from LIME~\cite{meka2018lime} to train an initial Normal-Net. Then we freeze Normal-Net and train Light-Net from scratch by our low-rank loss on reflectance, as the $1^{st}$ round joint training. Then Light-Net is frozen and Normal-Net is trained from the initial model by the same low-rank loss on reflectance. The joint training is driven by the Relit dataset, using 750k unlabeled images. Normal-Net is weakly-supervised due to the pretraining and all other nets are self-supervised. 
The joint training scheme effectively avoids the axis ambiguity and the quantitative ablation studies are shown in Section~\ref{exp:inv}. 

\subsection{Non-Lambertian object relighting}\label{sec:relighting}

After inverse rendering, an input photo is decomposed into normal, lighting, reflectance, shading and a possible specular component by our network. 
With these predicted properties, along with the lighting of new scenes, the object is re-rendered and inserted into new scenes. We propose a specular rendering layer in Section~\ref{sec:specrender}. Given specularity parameters (specular reflectance and smoothness), we can relight the object in a wide range of materials.  

Both diffuse and specular render layers take spherical harmonic coefficients as lighting inputs, which present low-frequency environment lighting. The transformations from HDR lighting paranomas to SH coefficients are pre-computed offline. 
We also implement a mobile App, whose details are in Appendix~\ref{sec:supp}. 

\subsubsection{Diffuse rendering layer}\label{sec:diffuserender}

In order to encode the shading rendering while keeping the whole network differentiable, we adopt a diffuse rendering layer respecting to Equation (\ref{equation:phong})-(\ref{equation:SH}), based on \cite{ramamoorthi2001efficient}. 
The rendering layer takes the spherical harmonic coefficients as lighting inputs. Combining Equation (\ref{equation:phong})-(\ref{equation:SH}), introducing  coefficients $\hat{A}_l$ from \cite{ramamoorthi2001relationship}, and incorporating normal into the spherical harmonic bases, the shading and the diffuse component of relit images are rendered by:
\begin{equation}
	\label{equation:sh2}
	I_d(p)=\mathcal{A}(p)\sum_{\omega \in \mathcal{L}}l_\omega(L_\omega\cdot n(p))=\mathcal{A}(p)\sum_{l,m}\hat{A}_l C_{l,m}Y_{l,m}(\theta,\phi),
\end{equation}

\noindent where $(\theta,\phi)$ is the spherical coordinates where $(x,y,z)=(\sin\theta\cos\phi, \sin\theta\sin\phi,\cos\theta)$ and $n(p)=(x,y,z)$.  


\subsubsection{Specular rendering layer}\label{sec:specrender}
\begin{figure}
	\centering
	\includegraphics[width= \linewidth]{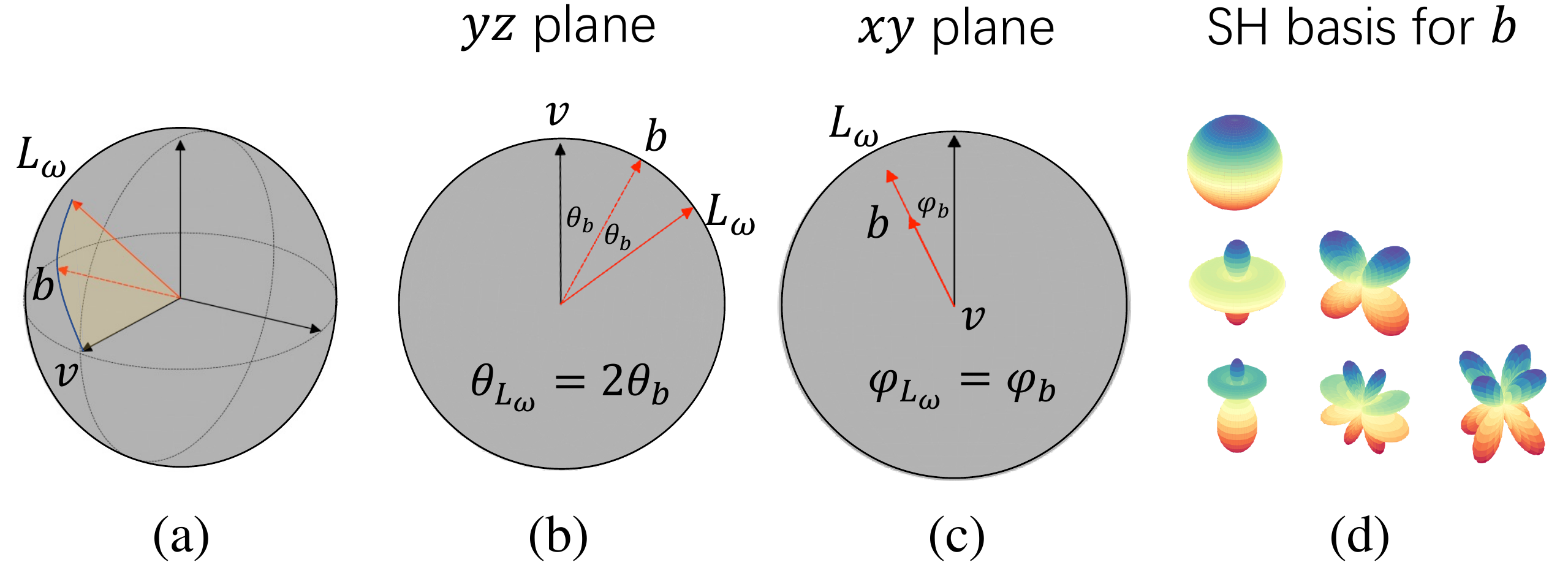}\
	\caption{The visual illustration of $b$, $L_\omega$ and $v$. (a) The view point $v$ is set as $[0, 0, 1]$ in our inverse rendering problem. $b$ is the bisector of $L_\omega$ and $v$. (b) Observe $b$, $L_\omega$ and $v$ in the $yz$ plane view. We find that the polar angle $\theta_{L_\omega}=2\theta_{b}$. (c) We find that the azimuth angle $\phi_{L_\omega}=\phi_{b}$ from the $xy$ plane view. (d) Then we get spherical harmonic basis $Y_{l,m}(\theta_{b},\phi_{b})$ for differentiable rendering of the specular component.}
	\label{fig:spec}
\end{figure}
Since the specular componet is view dependent, which can not be simply parameterized with $Y_{l,m}$ and $C_{l,m}$ as in the diffuse renderer. 
With the assumption of distant lighting, the view point is fixed. As shown in Figure~\ref{fig:spec}, $b = \frac{L_\omega+v}{\|L_\omega+v\|}$ is the bisector of light direction $L_\omega$ and view point $v$. Note that, $b$ has the same azimuth angle $\phi$ as $L_\omega$ while polar angle $\theta$ is only a half under spherical coordinate system as shown in Figure~\ref{fig:spec}. 
Since the predicted normal map has a pixel-to-pixel correspondence to the input image, which means the normal map is projected perspectively. We only need to apply orthogonal projection in the re-rendering step by assuming viewing the object the $z$ direction, which means $v = [0, 0, 1]$. The re-rendered images share a pixel-wise correspondence to observed images, following a perspective projection. 

Now we can modify $Y_{l,m}$ into $\hat{Y}_{l,m}$, and use $\hat{A}_l\hat{Y}_{l,m}$ to describe the distribution of all possible $b$ as well, keeping lighting coefficients $C_{l,m}$ unchanged for sharing between both renderers. 
\begin{align}
	\hat{Y}_{0,0}(\theta, \phi)&=Y_{0,0}(2\theta, \phi) = c_0\nonumber\\
	\hat{Y}_{1,1}(\theta, \phi)&= Y_{1,1}(2\theta, \phi) = c_1\sin 2\theta\cos\phi = 2c_1xz\nonumber\\
	\hat{Y}_{1,-1}(\theta, \phi)&=Y_{1,-1}(2\theta, \phi) = c_1\sin 2\theta\sin\phi = 2c_1yz\nonumber\\
	\hat{Y}_{1,0}(\theta, \phi)&=Y_{1,0}(2\theta, \phi) = c_1\cos 2\theta = c_1(2z^2-1)\nonumber\\
	\hat{Y}_{2,-2}(\theta, \phi) &=Y_{2,-2}(2\theta, \phi)= 4c_2xyz^2\nonumber\\
	\hat{Y}_{2,1}(\theta, \phi) &=Y_{2,1}(2\theta, \phi)= c_2(4xz^3-2xz)\nonumber\\
	\hat{Y}_{2,-1}(\theta, \phi) &=Y_{2,-1}(2\theta, \phi)= c_2(4yz^3-2yz)\nonumber\\
	\hat{Y}_{2,0}(\theta, \phi) &=Y_{2,0}(2\theta, \phi)= c_3(3(4z^4-4z^2+1)-1)\nonumber\\
	\hat{Y}_{2,2}(\theta, \phi) &=Y_{2,2}(2\theta, \phi) = c_4(4x^2z^2-4y^2z^2)\nonumber\\
	c_0&=0.282095,c_1=0.488603\nonumber
	\\
	c_2=1.092&548,c_3=0.315392,c_5=0.546274\nonumber
\end{align}

Hence, we can write the differentiable rendering approximation for the specular component similar as Equation~(\ref{equation:sh2}):
\begin{equation}
	\begin{aligned}
		H(p)&=s_p\sum_{\omega \in \mathcal{L}}l_\omega(\frac{L_\omega+v}{\|L_\omega+v\|}\cdot n(p))^\alpha \\
		&\approx s_p\sum_{l,m} C_{l,m}(\hat{A}_l\hat{Y}_{l,m}(\theta,\phi))^\alpha.
	\end{aligned}
\end{equation}





\section {Experiments}

In this section, we evaluate the performance of inverse rendering and image relighting. The inverse rendering evaluation with a series of state-of-the-art methods
is presented in Section~\ref{exp:inv}, along with several ablations. 
For image relighting, we provide quantitative evaluations on a synthetic dataset in Section~\ref{exp:relighting}, and real object insertion is demonstrated in Figure~\ref{fig:teaser} and the project page. 

\subsection{Inverse rendering}\label{exp:inv}
Many prior works address surface normal estimation or intrinsic image decomposition but not both, 
and there are no benchmark datasets for inverse rendering, 
we evaluate these two tasks individually. Evaluations on lighting and specularity are in Appendix~\ref{sec:supp}. 
The end-to-end inverse rendering takes 0.15 seconds per image at $256\times 256$ on a Titan T4 GPU.

\noindent{\bf{Intrinsic image decomposition.}} We compare our self-supervised intrinsic image decomposition to several inverse rendering methods (InverseRenderNet~\cite{yu2019inverserendernet}, RelightNet~\cite{Lichy_2021_CVPR}, ShapeAndMaterial~\cite{Lichy_2021_CVPR}), and intrinsic image decomposition methods~\cite{shi2017learning,yi2020leveraging,liu2020unsupervised,li2018cgintrinsics} on MIT Intrinsics dataset, which is a commonly-used benchmark dataset for IID. 
To evaluate the performances and cross-dataset generalizations, all methods are not finetuned on this dataset. 
We adopt scale-invariant MSE (SMSE) and local scale-invariant MSE (LMSE) as error metrics, which are designed for this dataset~\cite{grosse2009ground}. 
As shown in Table~\ref{table:MIT} (visual comparisons are in the supplementary material), our method outperforms all unsupervised and self-supervised methods and has comparable performance with supervised ones. 
Note that the assumptions of white illumination and Lambertian surfaces in this dataset fit the cases of synthetic data, which benefit supervised methods. However, self-supervised and unsupervised methods enable training on unlabeled real-image datasets, which produce better visual results on unseen natural images.
As shown in Figure~\ref{fig:endtoend}, 
SIRFS~\cite{barron2015shape}, a method based on scene priors, fails to decompose reflectance colors. InverseRenderNet~\cite{yu2019inverserendernet} and RelightingNet~\cite{yu2020self} tend to predict a similar color of shading and reflectance, leading to unnatural reflectance colors.  ShapeAndMaterial~\cite{Lichy_2021_CVPR} generates visually good results but has artifacts on reflectance due to specular highlights. Our method decomposes these components by considering non-Lambertian cases. 

\begin{figure*}
	\centering
	\includegraphics[width=\linewidth]{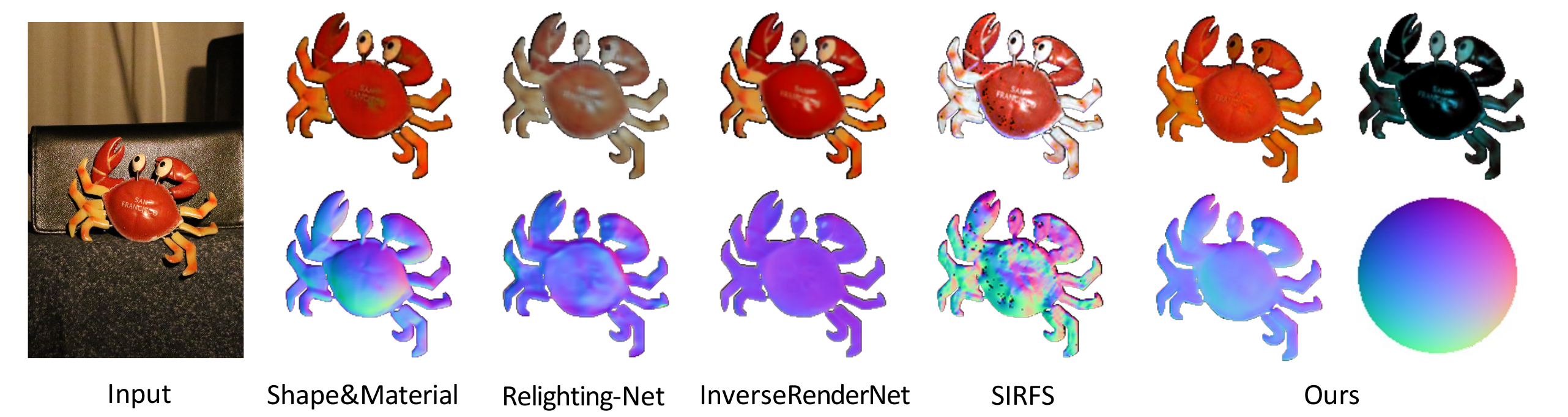}\
	\caption{Qualitative comparisons on an unseen image, comparing with state-of-the-art methods. The first row shows the reflectance of all methods and specular highlights of our method. The second row shows estimated normal maps and the colormap for reference. }
	\label{fig:endtoend}
\end{figure*}

\begin{figure*}
	\centering
	\includegraphics[width=\linewidth]{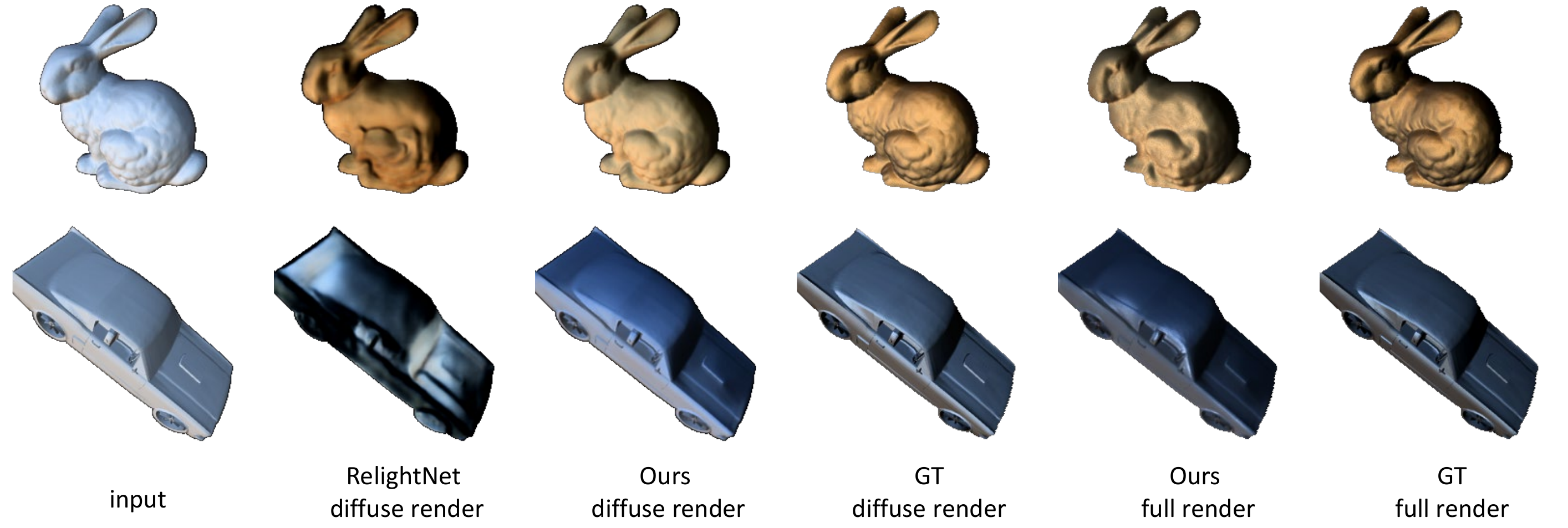}\
	\caption{Comparisons of object relighting with RelightNet~\cite{yu2020self} and ground truths. }
	\label{fig:relighting}
\end{figure*}

\begin{table}[t]
	\caption{Quantitative comparisons with state-of-the-art alternatives and ablation study of  intrinsic image decomposition on MIT intrinsic dataset. }\label{table:MIT}
	\vspace{0pt}
	\centering
	\scalebox{0.8}{
		\begin{tabular}{c c c c c } 
			\hline 
			Methods&Supervision&Data type&SMSE&LMSE\\
			\hline
			Shi et al.~\cite{shi2017learning}&Sup.& Synthetic&0.0194&0.0318\\
			Li et al.~\cite{li2018cgintrinsics}&Sup.&Synthetic&0.0186&\bf{0.0259}\\
			Shape\&Material~\cite{Lichy_2021_CVPR}&Sup.&Synthetic&\bf{0.0150}&0.0309\\
			\hline
			RelightingNet~\cite{yu2020self}&Self-sup.&Real&0.0368&0.1077\\
			Yi et al.~\cite{yi2020leveraging}&Unsup.&Real&0.0231&0.0422\\
			InverseRenderNet~\cite{yu2019inverserendernet}&Self-sup.&Real&0.0299&0.0855\\
			Liu et al.~\cite{liu2020unsupervised}&Unsup.&Real&0.0193&0.0428\\
			Ours&Self-sup.&Real&\bf{0.0186}&\bf{0.0369}\\
			\hline
			$1^{st}$ round training&Self-sup.&Real&0.0224&00420\\
			w/o joint training&Self-sup.&Real&0.0216&0.0399\\
			loss$^+$ ($\sigma_2$)&Self-sup.&Real&0.0357&0.0513\\
			loss* ($\sigma_2/\sigma_1$)&Self-sup.&Real&0.0808&0.2137\\
			\hline
	\end{tabular}}
\end{table}
\begin{table}[t]
	\caption{Quantitative comparisons with state-of-the-art alternatives and ablation study of surface normal estimation on the dataset from Janner et al.\cite{janner2017self}.}\label{table:janner}
	\centering
	\scalebox{0.8}{
		\begin{tabular}{c c c } 
			\hline 
			Methods&MSE&DSSIM\\
			\hline
			SIRFS~\cite{barron2013intrinsic}&0.0230&0.0243\\
			SVBRDF~\cite{li2018learning2}&0.0144&0.0278\\
			InverseRenderNet~\cite{yu2019inverserendernet}&0.0084&0.0272\\
			RelightNet~\cite{yu2020self}&0.0080&0.0265\\
			ShapeAndMaterial~\cite{Lichy_2021_CVPR}&0.0060&0.0228\\
			Ours&\bf{0.0054}&\bf{0.0201}\\
			\hline
			$1^{st}$ round training&0.0061&00219\\
			w/o joint training&0.0065&0.0228\\ 
			loss$^+$ ($\sigma_2$)&0.0059&0.0213\\
			loss* ($\sigma_2/\sigma_1$)&0.0083&0.0309\\
			\hline
			
	\end{tabular}}
\end{table}

\noindent{\bf{Normal estimation.}} We compare our method with several inverse rendering methods~\cite{barron2015shape,li2018learning2,yu2019inverserendernet,yu2020self,Lichy_2021_CVPR} on synthetic dataset from Janner et al.~\cite{janner2017self}. Since the dataset is too large (95k), and SIRFS~\cite{barron2015shape} takes one minute for each data, a testing set of 500 images is uniformly sampled, covering a wide variety of shapes. In Table~\ref{table:janner}, the evaluations are reported with two error metrics, MSE and DSSIM, measuring pixel-wise and overall structural distances. Our method yields the best performance. 
Qualitative comparisons are shown in Appendix~\ref{sec:supp}. 

{\noindent\bf{Ablations.}} 
We present ablations in the last four rows in Table~\ref{table:MIT}-\ref{table:janner}.  $1^{st}$ round training denotes the networks of initial Normal-Net and self-supervised Light-Net. ``w/o joint training'' denotes training Normal-Net and Light-Net simultaneously, rather than alternatively. Previous works propose different formulations of low-rank loss, as the second singular value ($\sigma_2$)~\cite{yi2018faces,zhu2020adacoseg} or the second singular value normalized by the first one ($\frac{\sigma_2}{\sigma_1}$)~\cite{yi2020leveraging} to enforce a matrix to be rank one. 
As discussed in the original papers, these losses are unstable in training and would degenerate to local optima. 
The proposed low-rank constraint is more robust as proven, not suffering from local optimas. More discussions and visual comparisons of these low-rank losses are in Appendix~\ref{sec:supp}. 

\begin{table}[t]
	\caption{Quantitative evaluation on relighting. 
	}\label{table:relighting} 
	\centering 
	\scalebox{0.7}{
		\begin{tabular}{c |c c| c c |c c} 
			\hline 
			&\multicolumn{2}{c|}{Baseline*}&\multicolumn{2}{c|}{RelightNet}&\multicolumn{2}{c}{Ours}\\
			&MSE&DSSIM&MSE&DSSIM&MSE&DSSIM\\
			\hline
			Diffuse &0.2210&0.1350&0.1144&0.0788&\textbf{0.0926}&\textbf{0.0616}\\
			With specularity&0.2152&0.1272&-&-&\textbf{0.0876}&\textbf{0.0720}\\
			\hline
		\end{tabular}
	}
	\vspace{-0.5cm}
\end{table}

\subsection{Image relighting}\label{exp:relighting}

After inverse rendering, a differentiable non-Lambertian renderer is used to relight the object under new lighting. 
The rendering is efficient, taking 0.35 seconds per image at $256\times 256$ on a single Tesla T4 GPU. 
For quantitative evaluations, we rendered an evaluation set of 100 objects under 30 lighting environments, with various materials. 
For each object, we use one image under one lighting as input, and relight it under the other 29 lighting for evaluation. We compare our method with a state-of-the-art method RelightNet~\cite{yu2020self}, which only provides diffuse relighting. To be fair, we compare them on diffuse relighting only. Ours is evaluated for both diffuse and non-Lambertian relighting. 
Comparisons are shown in Figure~\ref{fig:relighting} and Table~\ref{table:relighting}, more in Appendix~\ref{sec:supp}. Baseline* in the table denotes naive insertions without relighting.
Object insertion and App demos are on the project page, where our method relights and inserts objects into new scenes realistically. 
\vspace{-0.2cm}
\section{Conclusions}

We present a single-image relighting approach based on weakly-supervised inverse rendering, driven by a large foreground-aligned video dataset and a low-rank constraint. 
We propose the differentiable specular renderer for low-frequency non-Lambertian rendering. Limitations including shadows and parametric models are discussed in Appendix~\ref{sec:supp}. 
\paragraph{Acknowledgements. }  
We thank Kun Xu for the helpful discussions. We thank Sisi Dai and Kunliang Xie for data capturing. This work is supported in part by the National Key Research and Development Program of China (2018AAA0102200), NSFC (62002375, 62002376, 62132021), Natural Science Foundation of Hunan Province of China (2021JJ40696, 2021RC3071, 2022RC1104) and NUDT Research Grants (ZK22-52).
{\small
	\bibliographystyle{ieee_fullname}
	\bibliography{egbib}
}
\clearpage
\begin{appendices}
	
	\section{Supplementary material}\label{sec:supp}
	In this appendix, we introduce additional experiments, discussions, details of relighting video demos, Android app implementation, the Relit dataset, as well as network and training details. 
	
	\subsection{Mathematical proofs of Theorem~1}
	\begin{theorem}\label{theorem:rankone}
		\textbf{Optimal rank-one approximation.} By SVD,  $R = U\Sigma V^T$,  $\Sigma=diag(\sigma_1,\sigma_2,...\sigma_k)$, $\Sigma'=diag(\sigma_1,0,...)$, $\bar{R} = U\Sigma'V^T$ is the optimal rank-one approximation for R, which meets:
		
		\begin{equation}
			\label{equation:optimal}
			||\bar{R}-R||^2_F = \min_{\scriptscriptstyle b\in\mathcal{R}^{N},c\in\mathcal{R}^{d}}||bc^T-R||^2_F,
		\end{equation}
		
		\noindent where $||\cdot||_F$ denotes the Frobenius norm of a matrix. 

	\end{theorem}
	
	\begin{proof}
		The objective in (\ref{equation:optimal}) can be written as following:
		
		\begin{equation}
			\label{equation:optimal_sigma}
			||bc^T-R||^2_F = \sum_{i=1}^{d}||c_i\cdot b-r_i||^2_2. 
		\end{equation}
		
		\noindent To minimize $||c_i\cdot b-r_i||^2_2$ while $b$ is a fixed unit vector, $c_i\cdot b$ should be the projection of $r_i$ onto $b$ ($r_i$ is the $i^{th}$ column of $R$). It is equivalent to $c=b^TR$. Then we reduce the optimization problem (\ref{equation:optimal}) as:
		
		\begin{equation}
			\label{equation:reduced}
			\min_{\scriptscriptstyle b\in\mathcal{R}^{N},||b||_2=1}||bb^TR-R||^2_F.
		\end{equation}
		
		\noindent Since $b$ is a unit vector and  V are orthonormal, we can rewrite $||bb^TR||^2_F$ as:
		
		\begin{equation}
			\label{equation:svd_bbt}
			\begin{aligned}
				||bb^TR||^2_F&=||b^TR||^2_2=||b^TU\Sigma V^T||^2_2\\
				&=||b^TU\Sigma||^2_2=\sum_{i=1}^k(b^Tu_i)^2\sigma_i^2.
			\end{aligned}	
		\end{equation}
		
		\noindent By Pythagorean Theorem, and $||R||^2_F\geq  {\textstyle \sum_{i=1}^{n}} ||c_ib||_2$, optimization problem (\ref{equation:reduced}) is equivalent to maximizing (\ref{equation:svd_bbt})). Since $\sum_{i=1}^k(b^Tu_i)^2=1$ and $\{\sigma_i\}$ are descending, Equation (\ref{equation:svd_bbt}) is maximized when $(b^Tu_1)^2=1$. It can be accomplished by setting $b=u_1$ and $c=\sigma_1v_1^T$, i.e., $\bar{R} = U\Sigma'V^T = bc^T$ is the optimal rank 1 approximation for $R$. 
	\end{proof}

	\subsection{Additional experiments}

	\subsubsection{Evaluation of Light-Net}
	
	To evaluate the performance of Light-Net, we randomly sampled a testing set of 200 images from LIME~\cite{meka2018lime}. It is a synthetic dataset of Bigbird~\cite{singh2014bigbird} and ShapeNet~\cite{chang2015shapenet} objects with ground truth normal, albedo, and shading. 
	We use lighting coefficients predicted by Light-Net to render shading with ground truth normal maps. By comparing the rendered shading and ground truths, we can evaluate the accuracy of the estimated lighting coefficients. 
	For quantitative evaluation, we adopt three metrics, including MSE, scale-invariant MSE, and SSIM. The results are in Table~\ref{table:lightnet}. We compare to two ablations from Table~\ref{table:MIT}-\ref{table:janner}, which are ``loss+'' and ``without joint training''. We can see that for lighting evaluation, our final model produces the lowest MSE and scale-invariant MSE, and comparable SSIM to ``without joint training''. From visual examples in Figure~\ref{fig:lighteval}, our model renders similar shading with ground truths, while the predicted lighting is more directional than ground truths. 
	
	Although part of the LIME dataset is used in the pretraining of Normal-Net, here we only use Light-Net for this evaluation, for which the dataset is completely unseen.

	\begin{figure}
		\centering
		\includegraphics[width=1\linewidth]{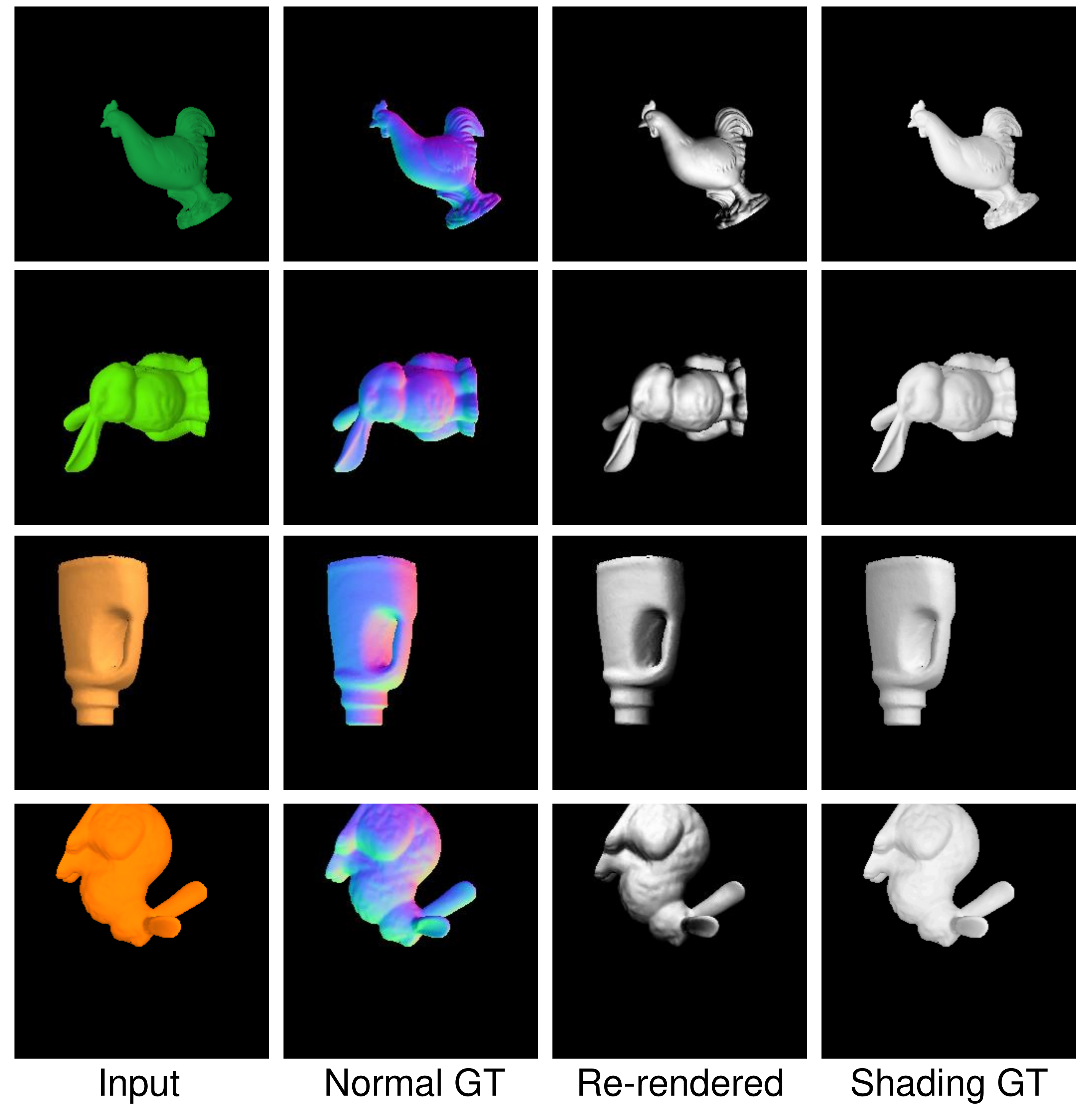}\
		\caption{Visual examples of Light-Net evaluation. We render shading from ground truth normal and predicted lighting, and produce close results with ground truth shading. }
		\label{fig:lighteval}
	\end{figure}

	\begin{table}[h]
		\caption{Quantitative evaluation of Light-Net. 
		}
		\centering 
		\scalebox{0.8}{
			\begin{tabular}{c |c c c } 
				\hline 
				&The final model&loss+&w/o joint training\\
				\hline
				MSE $\downarrow$&\textbf{0.0403}&0.0452&0.0414\\
				SMSE $\downarrow$&\textbf{0.0336}&0.0368&0.0345\\
				SSIM $\uparrow$&0.8684&0.8652&\textbf{0.8686}\\
				\hline
			\end{tabular}
		}
		\label{table:lightnet} 
	\end{table}
	
	\subsubsection{Evaluation of Spec-Net}
	To evaluate the performance of specular highlight extraction of Spec-Net, we compare with several prior methods on a real-image dataset from \cite{yi2020leveraging}. As shown in Table~\ref{table:specnet}, Spec-Net outperforms other methods in both SMSE and DSSIM. Visual comparisons are in Figure~\ref{fig:specnet}. On real images where highlights are strong, and highlight regions are saturated, most methods tend to over-extract specular highlights, while the Spec-Net performs well due to the training on a large scale of real images. 
	
	\begin{table}
		\caption{Quantitative evaluation of specular highlight separation on real images. 
		}
		\centering 
		\scalebox{1}{
			\begin{tabular}{c |c c c c} 
				\hline 
				&\cite{shen2013real}&\cite{shi2017learning}&\cite{yamamoto2019general}&Ours\\
				\hline
				MSE&0.0334&0.0305&0.0334&\textbf{0.0148}\\
				DSSIM&0.1745&0.2087&0.1743&\textbf{0.1500}\\
				\hline
			\end{tabular}
		}
		\label{table:specnet} 
	\end{table}

	\begin{figure}
		\centering
		\includegraphics[width=1\linewidth]{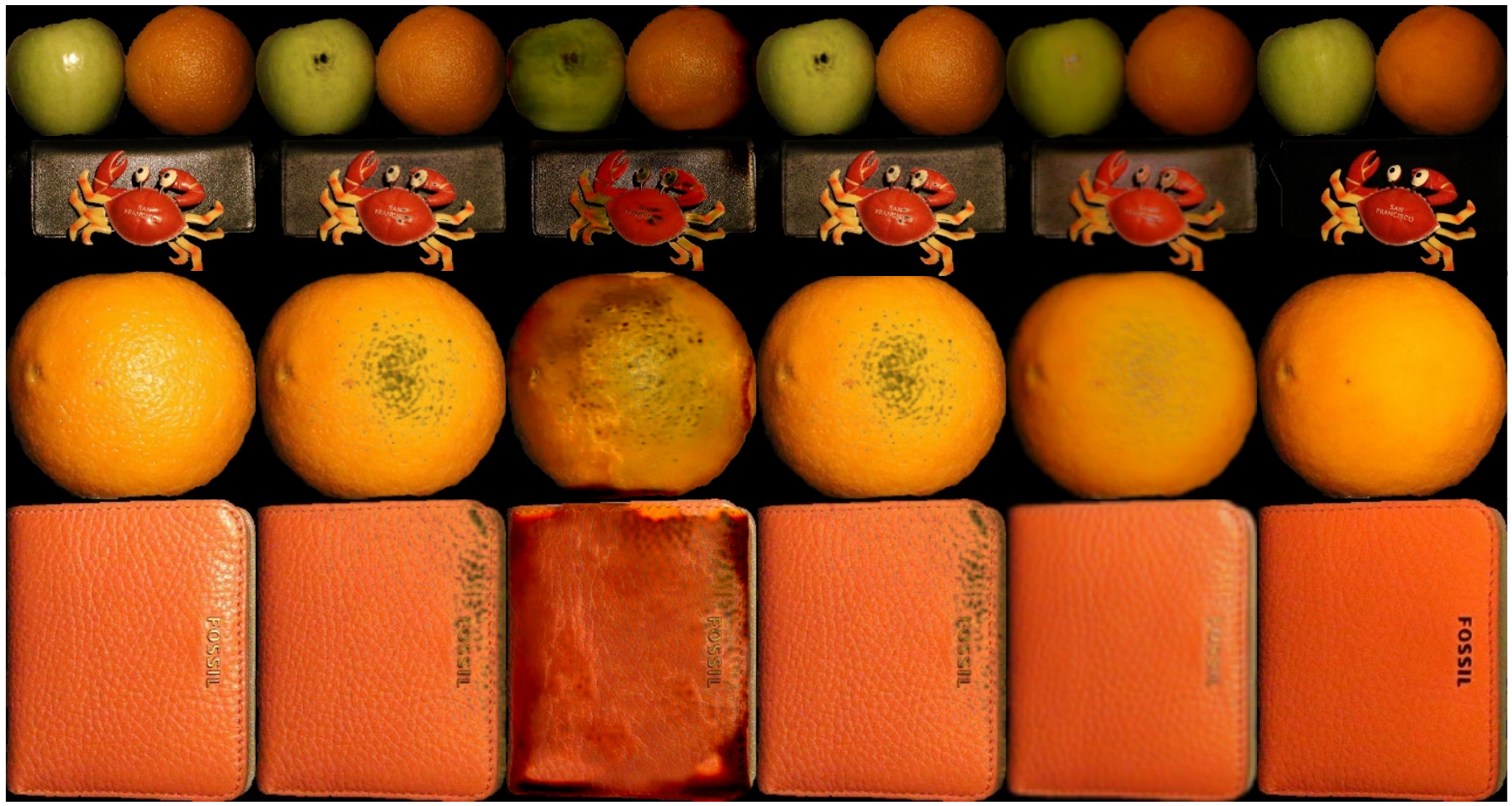}\
		\caption{Qualitative comparisons on four data from real-image specularity separation dataset from \cite{yi2020leveraging}, captured by cross-polarization. From left to right, there are input images and diffuse components after removing specular highlights by \cite{shi2017learning,shen2013real,yamamoto2019general}, ours, and ground truths.  }
		\label{fig:specnet}
	\end{figure}
	
	\subsubsection{Additional results of experiments in the main paper}
	
	We show additional results for experiments in the main paper. In Figure~\ref{fig:mit}, there are visual comparisons of two data from MIT intrinsics~\cite{grosse2009ground}. Here all methods are not fine-tuned on MIT dataset. Here SIRFS~\cite{barron2013intrinsic} and DI~\cite{narihira2015direct} are supervised methods. Yi~\cite{yi2020leveraging} and ours are self-supervised, while they predict shading by a Shading-Net, and our shading is rendered from predicted normal and lighting. In Figure~\ref{fig:normal}, there are visual comparisons of normal estimation to several state-of-the-art methods, for the quantitative evaluation in Table \ref{table:janner}, on unseen data from Janner et al.~\cite{janner2017self}. Our method produces more details in normal maps. In Figure~\ref{fig:insertion}, we compare with a full relighting pipeline RelightingNet~\cite{yu2020self} on real object insertion. 
	
	\begin{figure*}
		\centering
		\includegraphics[width=1\linewidth]{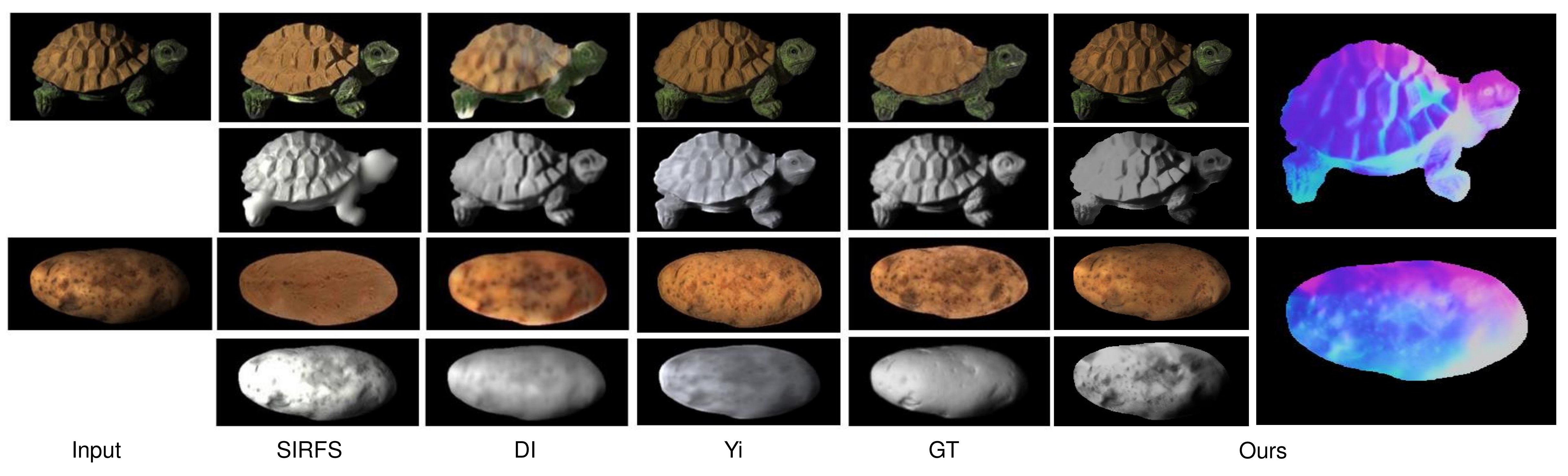}\
		\caption{Qualitative comparisons on two data from MIT Intrinsics. Odd rows are input images, albedos and even rows are shadings. The normal predicted by our method is shown at right.  }
		\label{fig:mit}
	\end{figure*}
	
	\begin{figure}
		\centering
		\includegraphics[width=1\linewidth]{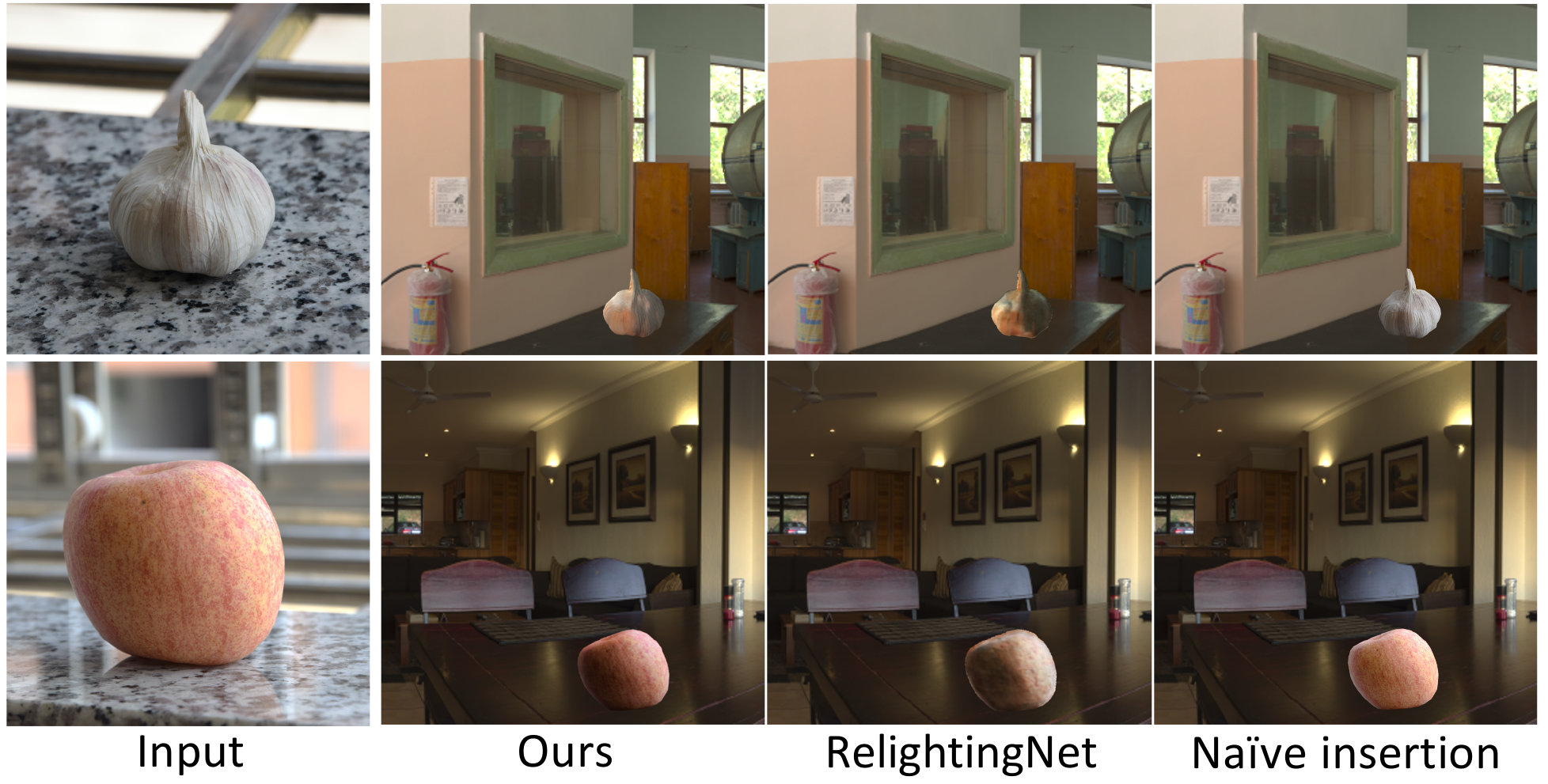}\
		\caption{Qualitative comparisons on object insertion by ours, RelightingNet~\cite{yu2020self} and naive insertion without relighting. }
		\label{fig:insertion}
	\end{figure}
	
	\begin{figure*}
		\centering
		\includegraphics[width=1\linewidth]{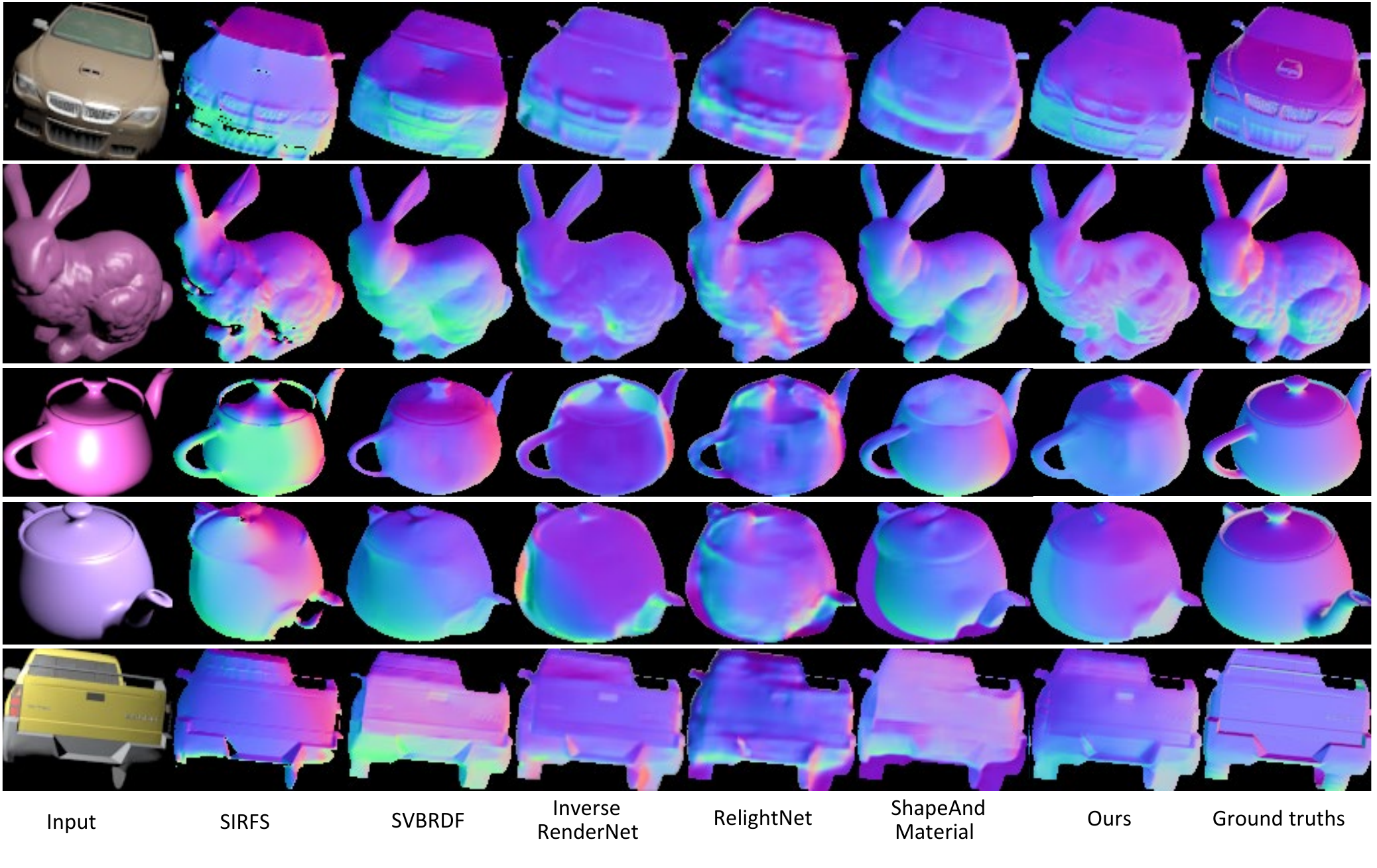}\
		\caption{Normal estimation comparisons with SIRFS\cite{barron2013intrinsic}, SVBRDF\cite{li2018learning2}, InverseRenderNet\cite{yu2019inverserendernet}, RelightNet\cite{yu2020self}and ShapeAndMaterial\cite{Lichy_2021_CVPR} on selected data from Janner et al.~\cite{janner2017self}. The reference color map can be found in Figure~\ref{fig:endtoend}, where the red channel is x-axis pointing right, green channel corresponds to y-channel pointing down, and the blue channel is z-axis pointing out from the image plane. }
		\label{fig:normal}
	\end{figure*}
	
	
	In Figure~\ref{fig:relight}, we provide supplementary results of Table~\ref{table:relighting} and Figure ~\ref{fig:relighting}. Our diffuse and non-Lambertian rendering layers produce similar results with GT renderers. GT renderers are implemented by Monte-Carlo sampling of point lights following the Blinn-Phong model.

	\subsubsection{More discussions}
	
	\noindent{\bf{Multi-view stereo as normal supervision.}} Previous method \cite{yu2019inverserendernet} uses multi-view stereo to reconstruct normal maps on outdoor building images in MegaDepth dataset\cite{li2018megadepth}, where ground truth depth maps are also available. Features on outdoor buildings are rich, which are suitable for multi-view stereo to reconstruction. 
	
	For object images, we explored similar approaches and found it not working for our scenarios. we use a reconstruction pipeline of adopting VisualSFM\cite{wu2011visualsfm} to reconstruct sparse point clouds, then PMVS2\cite{furukawa2010accurate} to further reconstruct dense point clouds. Applying the pipeline needs multi-view images as inputs, which would introduce a heavy workload for capturing multi-view images for all objects. For demonstration, we capture additional multi-view images and test the pipeline on several objects. For each object, we capture about 50 multi-view images as inputs. From the results, we find the point clouds are very sparse due to lack of features. A example is shown in Figure~\ref{fig:sfm}, textureless regions are quite common on natural objects, where the features are sparse, and reconstruction results have many holes on the resulting dense point clouds. For some other object, due to the lack of features, VisualSFM even fails to reconstruct a initial point cloud. Thus, adopting SFM and MVS to reconstruct geometry is not an option for our cases.

	\begin{figure}
		\centering
		\includegraphics[width=1\linewidth]{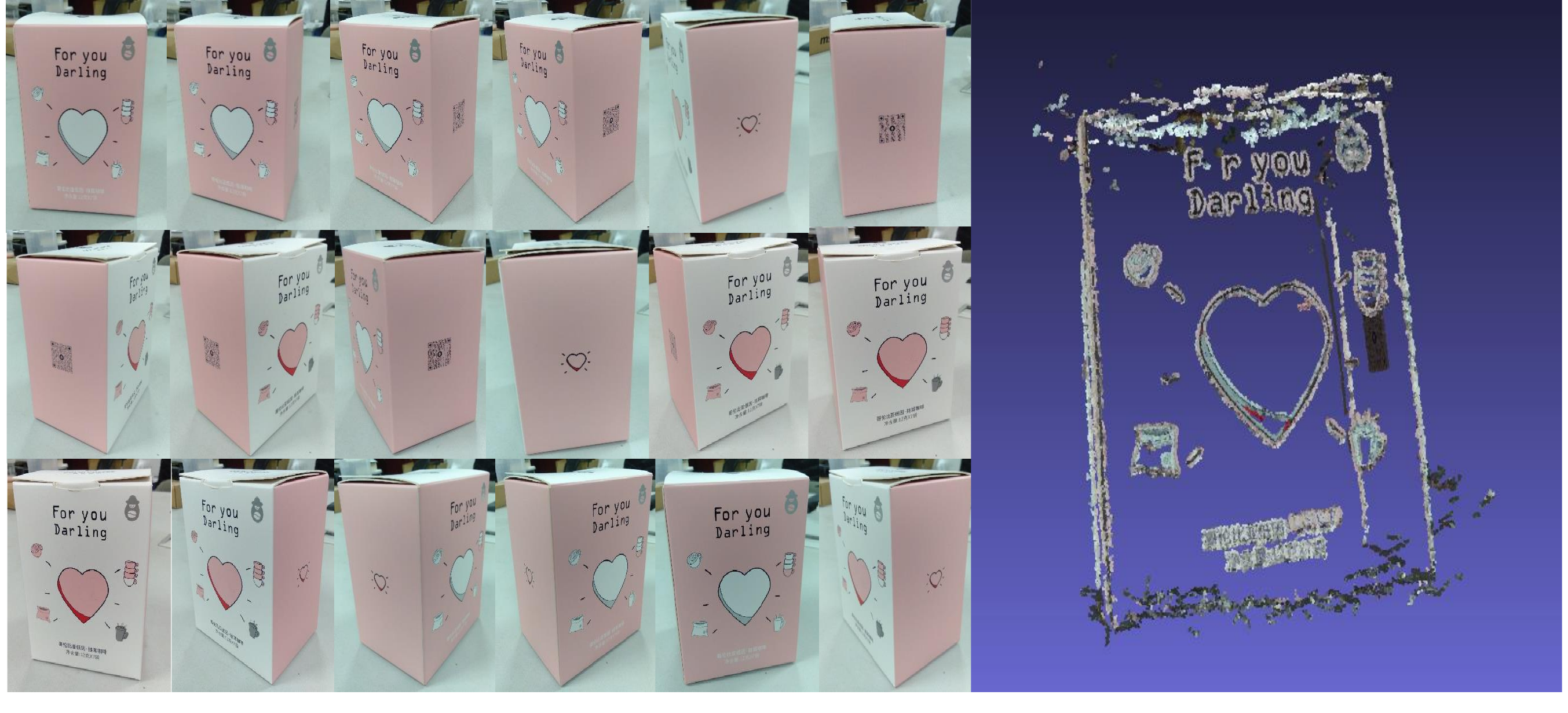}\
		\caption{Object reconstructed by VisualSFM and PMVS2. Selected multi-view image inputs are shown on the left and reconstructed dense point clouds are on the right. }
		\label{fig:sfm}
		\vspace{-0.3cm}
	\end{figure}
	
	\begin{figure}
		\centering
		\includegraphics[width=1\linewidth]{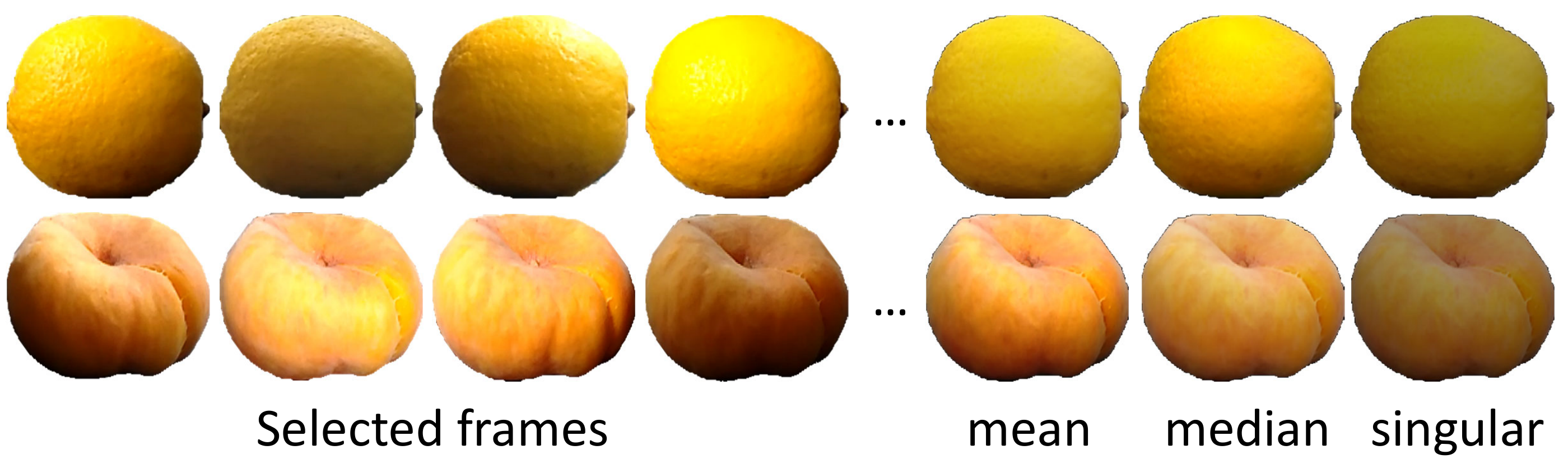}\
		\caption{On each row, selected images from one batch are shown at the left. Corresponding mean image, median image and singular image are at the right.  }
		\label{fig:meanimage}
	\end{figure}

	\noindent{\bf{Using median or mean reflectance vs. the singular reflectance.}} One may wonder whether using median or mean images of reflectance predictions in one batch will have similar results with our low-rank constraint. Firstly, losses between the median or mean reflectance of one batch and predicted reflectance are not scale-invariant. Secondly, the median image is not differentiable. Thirdly, we perform a large amount of testing on our Relit dataset and found that singular reflectance is more robust to shadows, intensity saturations and uneven lighting, which are common cases in natural images. Some visual comparisons are shown in Figure~\ref{fig:meanimage}, we can see that mean image may generate incorrect reflectance in some regions due to the above reasons while dominant singular reflectance generates much more reasonable reflectance maps. It is because SVD solves the dominant direction of reflectance maps, better than naive averaging. Note that we show cases on input images in Figure~\ref{fig:meanimage} because at the beginning of joint training, the network initializes from predicting reflectance the same as input images. We can see that using singular reflectance is much better visually, with convergence proven. \\
	
	\noindent{\bf{Comparisons to other low-rank losses. }}
	As mentioned in Section~\ref{exp:inv}, our definition of low-rank constraint is more robust and easy to converge. We evaluate the robustness of our low-rank loss with losses from \cite{yi2018faces} and \cite{yi2020leveraging}. Previous low-rank losses have more than one local optima as mentioned in \cite{yi2020leveraging}. Thus they have to use a pretraining phase to initialize the training, and the learning rates are hand-picked to make sure the final models converge to the local optima near the pretraining results. In Table~\ref{table:lowrankloss}, we found the learning rate has to be tuned carefully. For loss$^+$ in the table, a learning rate smaller than $10^{-8}$ would work. For loss*, we test learning rates from $10^{-2}$ to $10^{-8}$, and all cases degenerate to predict all-white or all-zero shadings. Setting a small learning rate also makes the training time much longer. Our loss has only one global and local optima, and it is promised to converge, and it does not suffer from degenerating. 
	
	Visual comparisons to previous low-rank losses (loss+ and loss*) from \cite{yi2018faces,yi2020leveraging} are in Figure~\ref{fig:losses}. We can see that loss+ gives similar results to ours, while albedo by our method is more smooth in color, and our normal is more accurate from Table~\ref{table:janner}. Note that here loss+ is trained in a small learning rate of $10^{-8}$ to prevent degeneration. It also benefits from our large-scale Relit dataset. However, even by a small learning rate of $10^{-8}$, loss* still degenerates and starts to predict all black albedo maps, as in Figure~\ref{fig:losses}. 
	
	\begin{figure*}
		\centering
		\includegraphics[width=1\linewidth]{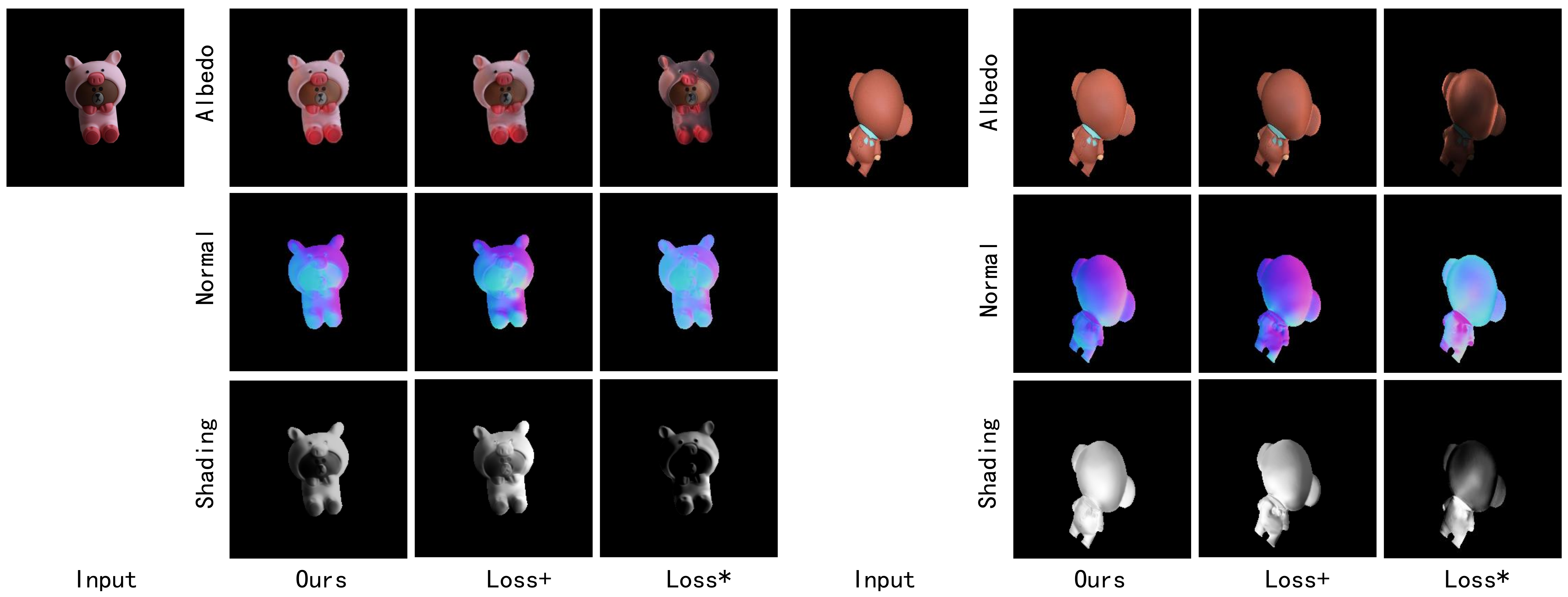}\
		\caption{Visual comparison of three low-rank losses on two unseen images. Benefiting from the Relit dataset, loss* produces similar results while setting a small learning rate. }
		\label{fig:losses}
	\end{figure*}

	\begin{figure*}
		\centering
		\includegraphics[width=1\linewidth]{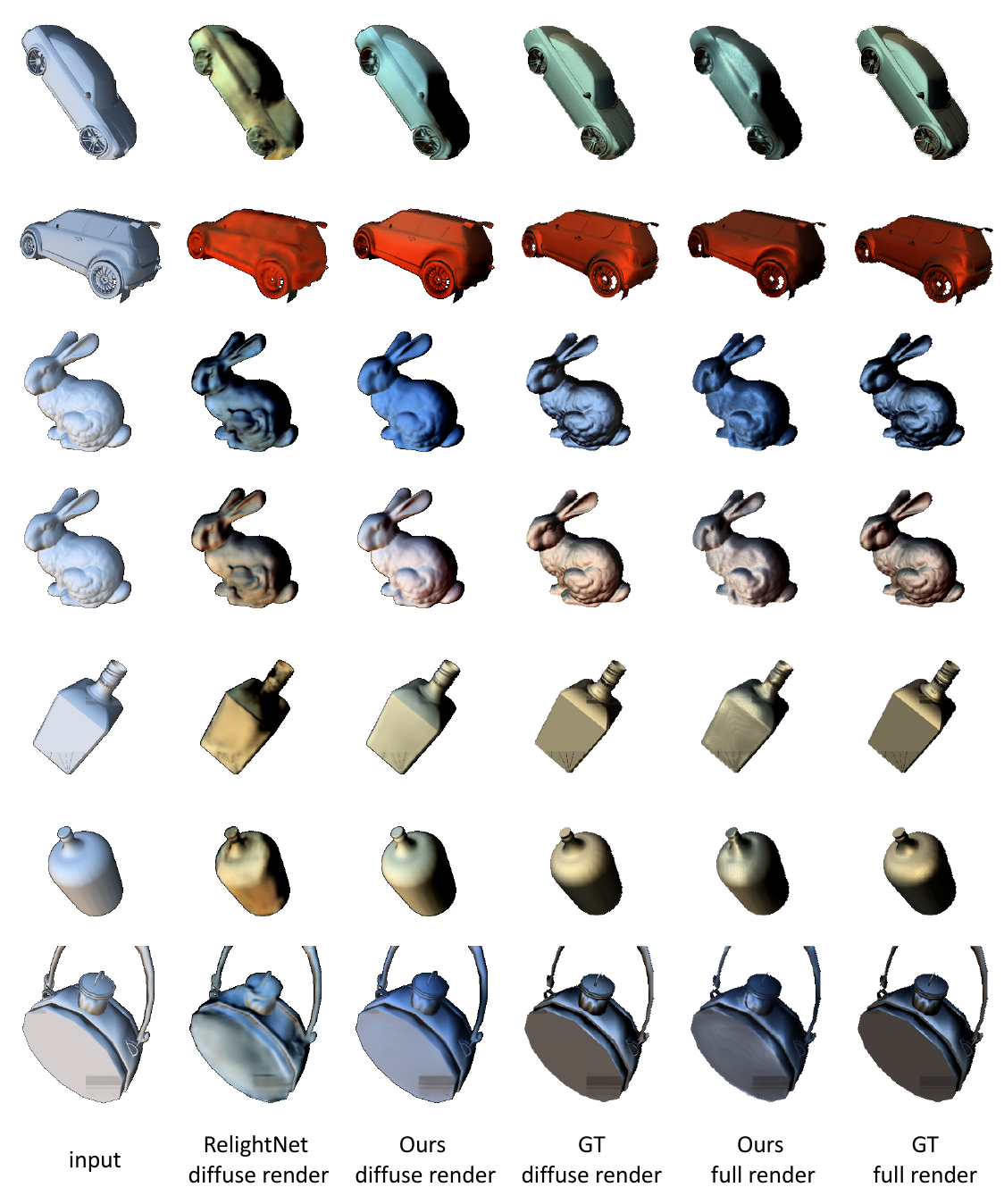}\
		\caption{Quantitative evaluation on relighting with and without specularity. RelightNet~\cite{yu2020self} can only provide diffuse relighting. Baseline* denotes the images under the original lighting. }
		\label{fig:relight}
	\end{figure*}
	\begin{table}[t]
		\centering 
		\scalebox{0.9}{
			\begin{tabular}{c| c c c c} 
				\hline 
				&$10^{-2}$&$10^{-4}$&$10^{-6}$&$10^{-8}$\\
				\hline
				loss$^+$ ($\sigma_2$)&\XSolidBrush&\XSolidBrush&\XSolidBrush&\Checkmark\\
				
				loss* ($\sigma_2/\sigma_1$)&\XSolidBrush&\XSolidBrush&\XSolidBrush&\XSolidBrush\\
				Ours&\Checkmark&\Checkmark&\Checkmark&\Checkmark\\
				\hline
			\end{tabular}
		}
		\caption{The robustness of different loss formulations. \XSolidBrush means the training degenerates to an invalid shading and \Checkmark means the training is converging.  }\label{table:lowrankloss} 
	\end{table}
	\subsection{Limitations}
	There are several limitations, as well as future directions of the proposed method. 
	One limitation is that, cast shadows (visibility) are not considered, which can further narrow the gap between relighting results and reality. 
	Furthermore, parametric models such as Blinn-Phong and Phong are difficult to model semitransparent and transparent materials, which are also common in real scenarios. Spherical harmonics are also limited to model high-frequency lighting components. We plan to explore these directions in the future. 
	
	\subsection{Relighting demos}
	
	On the project page \footnote{\href{https://renjiaoyi.github.io/relighting/}{https://renjiaoyi.github.io/relighting/}}, we include many relighting videos under changing backgrounds. Relit images are inserted to target scenes to show a seamless AR object insertion effect. We demonstrate single-object insertion and multi-object insertion where multiple objects are from different input images. We also demonstrate editing the materials of objects. Object insertion is quite popular in AR applications, and most AR Apps simply adopt naive insertion without relighting, such as the dancing hotdog in SnapChat, and furniture in Ikea Place. From the video, we can see our method generates much better object insertion results than naive insertion, demonstrating the importance of this problem. 
	
	Note that the backgrounds are cropped from HDR lighting panoramas, after Gamma corrections with $\gamma$ as $2.2$. Codes for pre-computation of Spherical Harmonic coefficients, and end-to-end inverse rendering and relighting will be released on the project page. 
	
	\subsection{App implementation}
	
	To implement the object relighting app in the Android mobile system, we convert the network models to Pytorch Mobile and package them inside the application as assets. For object photos captured from the camera, an on-device GrabCut in OpenCV is applied to obtain the object mask. To ensure acceptable automatic segmentation results, we require users to capture the objects under a background of solid colors. For photos loading from memory, the object mask is required as an additional input. We can insert and relight single or multiple objects from different photos into the same scene, and manipulate the layouts and sizes through simple dragging, tailored for amateur users. 
	
	The application is implemented in Java, using the Android Gradle plugin of version 3.5.0 with several additional Gradle and Pytorch dependencies. The app demo video is also on the project page. 
	
	\subsection{The Relit dataset}\label{sec:dataset}
	
	To capture foreground-aligned videos of objects under changing illuminations, we design an automatic device for data capture, as shown in Figure~\ref{fig:dataset} (left). The main part is an electric turntable painted black to avoid strong reflections. 
	While capturing data, objects and the camera are fixed on the turntable. The turntable rotates at a uniform angular velocity of $12.6$ rad/s, controlled by a remote to avoid shaking. For each video, the device is rotated by $360^{\circ}$ for 50 seconds. 
	
	The device is chargeable and portable, enabling us to capture data under arbitrary scenes easily. The target object stays static in the image coordinate system in captured videos, with changing illuminations and backgrounds. These foreground-aligned videos can facilitate many tasks, such as image relighting, segmentation, and inverse rendering. 
	
	In summary, the Relit dataset consists of 500 videos for more than 100 objects under different indoor and outdoor lighting. Each video is 50 seconds, resulting in 1500 foreground-aligned frames under various lighting. In total, the Relit dataset consists of $750K$ images. 
	In pre-processing, we segment the mask for one frame of each video and apply it to all frames to remove the changing backgrounds. 
	Selected objects are shown in Figure~\ref{fig:dataset} (right) The objects cover a wide variety of shapes, materials, and textures. 
	
	Some foreground-aligned images in Relit dataset are shown in Figure~\ref{fig:dataset4}-\ref{fig:dataset7}. These are selected frames from some videos after preprocessing. Sample videos from the dataset are shown on the project page, where the device is very stable, making sure the foreground objects are staying well-aligned among all frames. The dataset is released on the project page.  
	
	\subsection{Network structure and training details}
	
	Normal-Net and Light-Net are the only two learnable modules in our diffuse pipeline, and an optional specular branch may be used depending on the materials of target objects. The structures are in Figure~\ref{fig:structure}. Spec-Net shares the same structure with \cite{yi2020leveraging}. The network to regress specular reflectance $S_p$ and smoothness $\alpha$ shares the same structure of Light-Net, while changing the output to 4 channels (3 for specular reflectance and 1 for smoothness). 
	
	In pretraining of Normal-Net, $50K$ synthetic images from LIME~\cite{meka2018lime} are used for training. The learning rate is $10^{-4}$ without further adjustments. The training lasts for 50 epochs, by Adam optimizer. 
	
	In our joint training, we use the large-scale foreground-aligned images from Relit dataset. Light-Net is initialized from scratch and Normal-Net is initialed by the pre-trained model. The learning rate is $10^{-6}$ without further adjustments. Each round of joint training last for 3 epochs, taking 60 minutes per epoch on Tesla P40 GPU. The joint training process driven by the proposed low-rank loss converges rapidly, which takes 6 hours in total, thanks to the convergence proven in Section~\ref{sec:unsupervised}. 
	
	
	\begin{figure*}
		\centering
		\includegraphics[width=\linewidth]{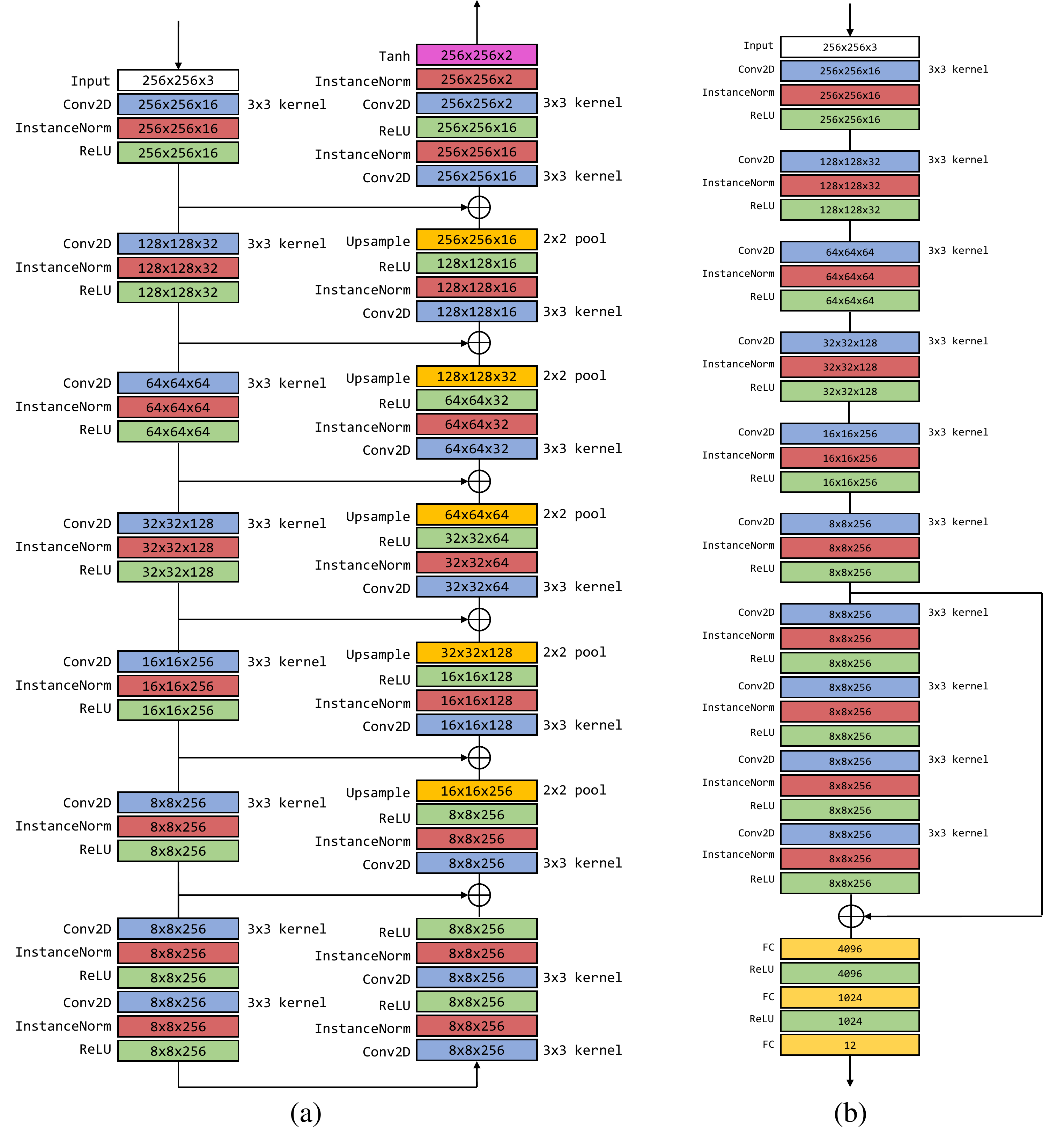}\
		\caption{(a) Structure of Normal-Net. (b) Structure of Light-Net.}
		\label{fig:structure}
	\end{figure*}
	
	

	
	\begin{figure*}
		\centering
		\includegraphics[width=0.8\linewidth]{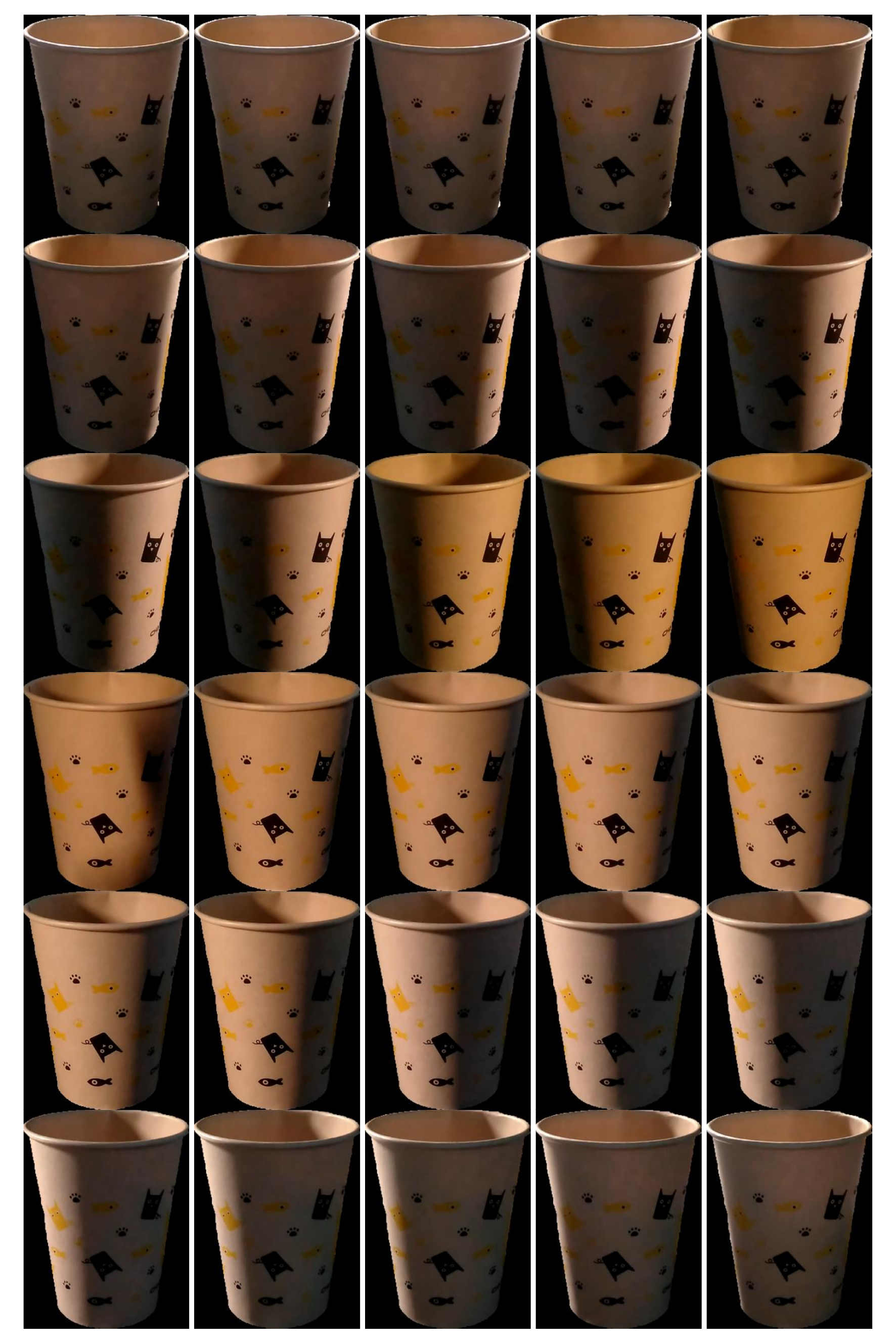}\
		\caption{Selected frames from one video in Relit dataset. }
		\label{fig:dataset4}
	\end{figure*}
	
	\begin{figure*}
		\centering
		\includegraphics[width=0.8\linewidth]{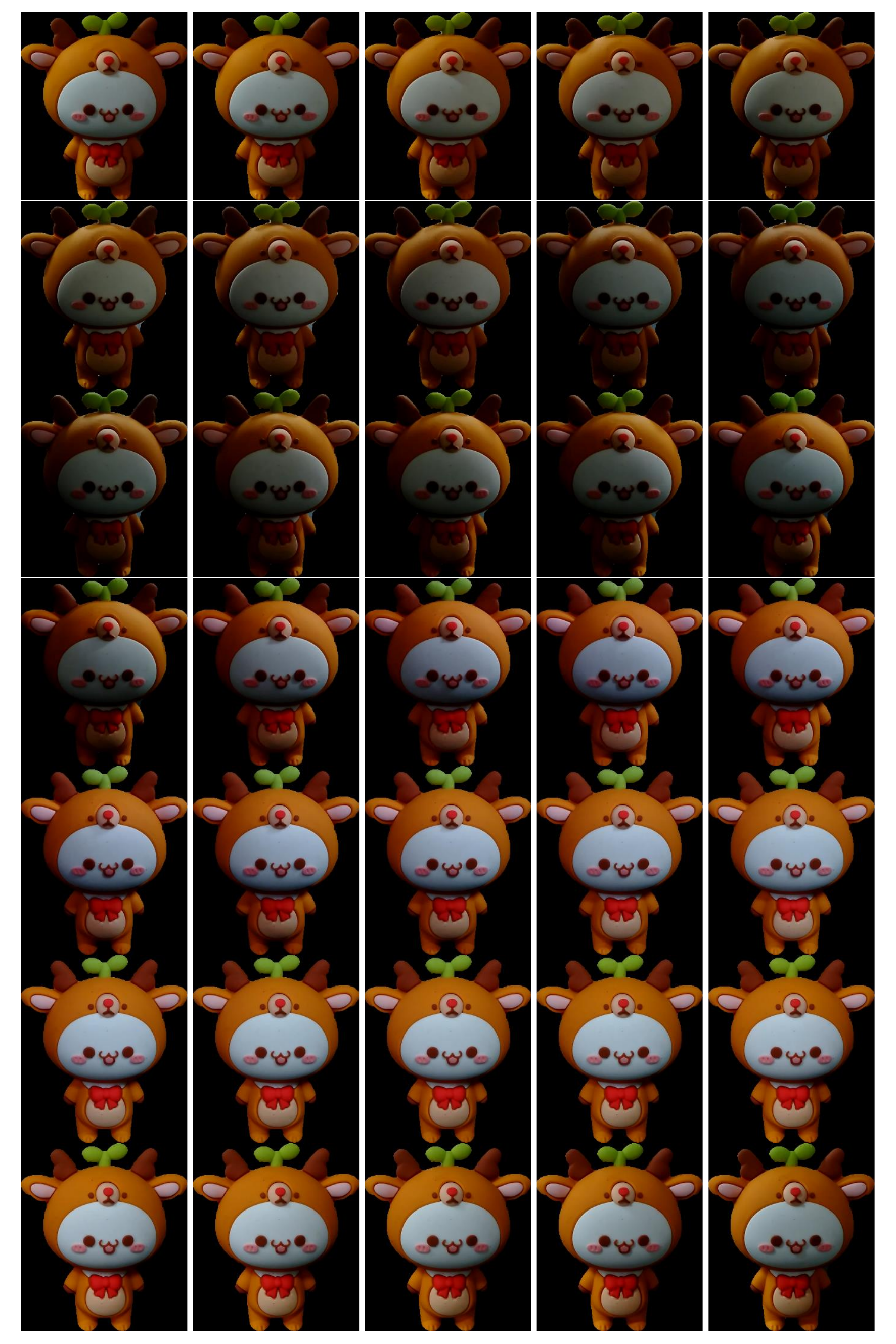}\
		\caption{Selected frames from one video in Relit dataset. }
		\label{fig:dataset5}
	\end{figure*}
	
	\begin{figure*}
		\centering
		\includegraphics[width=\linewidth]{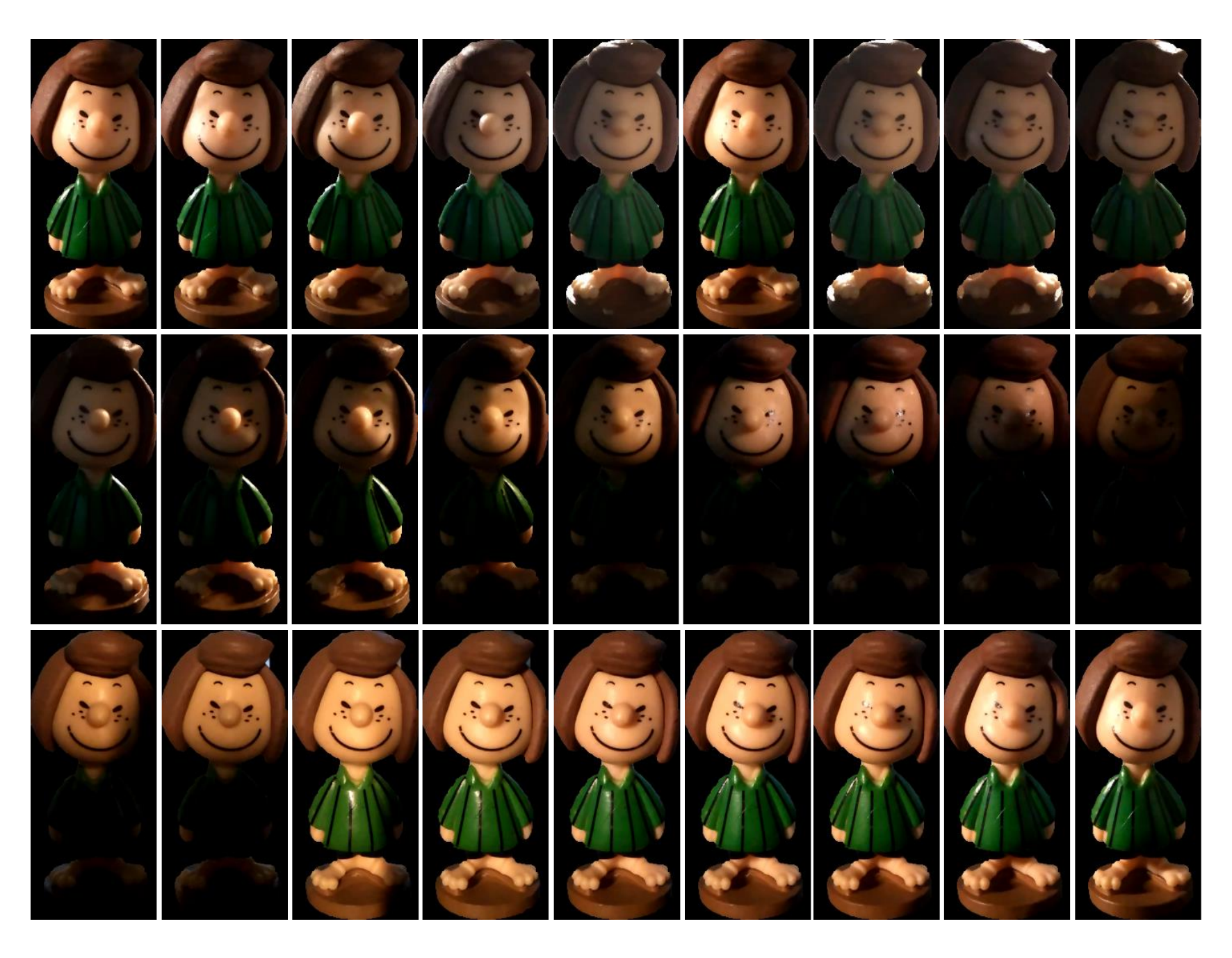}\                
		\caption{Selected frames from one video in Relit dataset. }
		\label{fig:dataset6}
	\end{figure*}
	
	\begin{figure*}
		\centering
		\includegraphics[width=\linewidth]{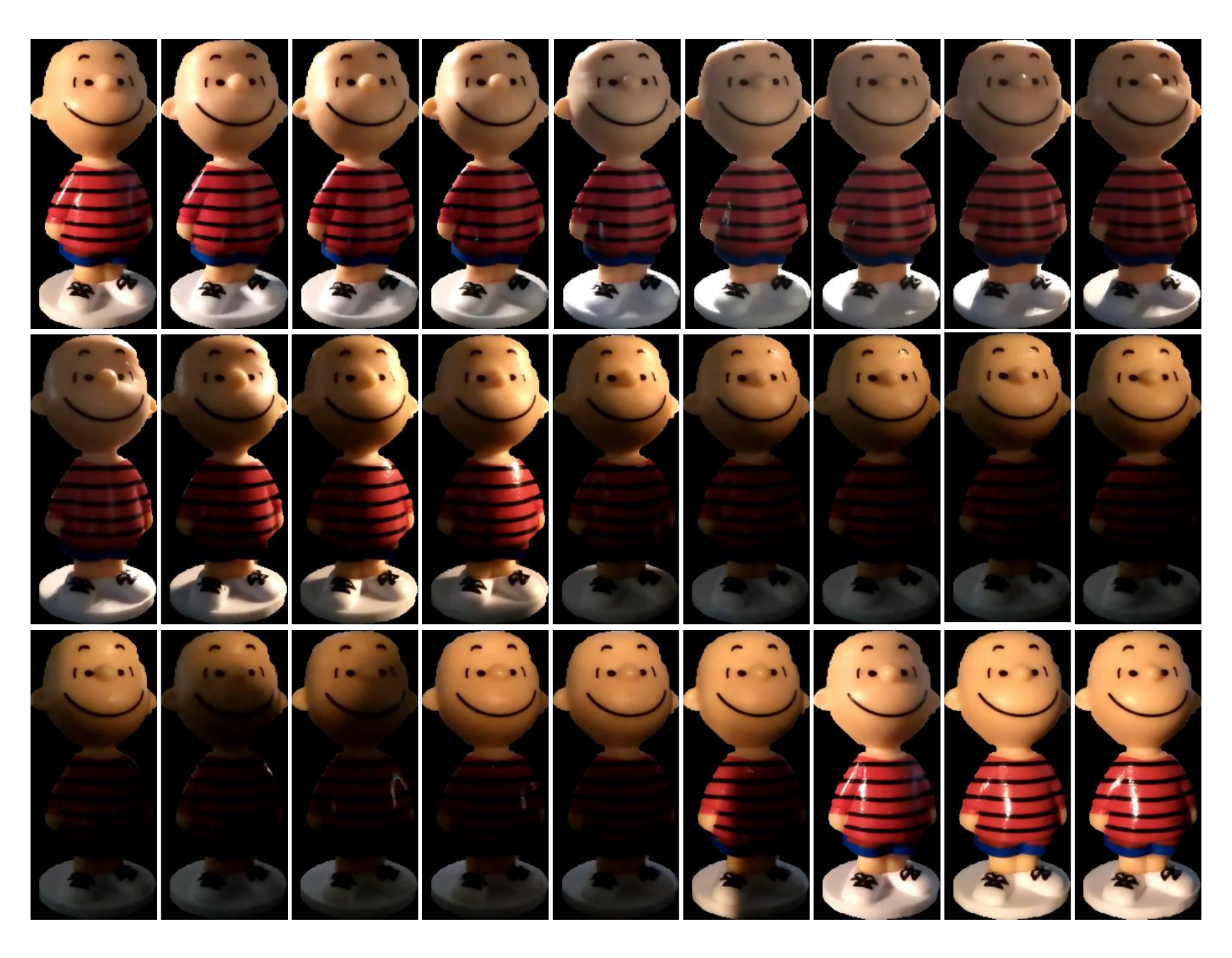}\
		\caption{Selected frames from one video in Relit dataset. }
		\label{fig:dataset7}
	\end{figure*}
	
\end{appendices}
\end{document}